\documentclass[11pt,a4paper]{article}

\usepackage{moreverb,url}
\usepackage{natbib}

\usepackage[colorlinks,bookmarksopen,bookmarksnumbered,citecolor=red,urlcolor=red]{hyperref}

\newcommand\BibTeX{{\rmfamily B\kern-.05em \textsc{i\kern-.025em b}\kern-.08em
T\kern-.1667em\lower.7ex\hbox{E}\kern-.125emX}}

\usepackage{tikz}
\usetikzlibrary{arrows,decorations.pathmorphing,backgrounds,positioning,fit,petri}
\usetikzlibrary{calc,intersections,through,backgrounds}
\usepackage{subcaption}

\usepackage[linesnumbered,lined,boxed,commentsnumbered]{algorithm2e}
\usepackage{booktabs}
\usepackage{lscape}

\usepackage{pstricks}
\usepackage{pst-node}
\usepackage{picinpar}
\usepackage{siunitx}
\usepackage{amsthm}
\usepackage{amssymb,amsmath,a4wide,mdwlist, mathtools}
\usepackage{enumerate}
\usepackage{algorithmic}
\usepackage{graphicx}
\usepackage{xcolor}
\usepackage{url}
\usepackage{rotating}

\SetCommentSty{mycommfont}

\usetikzlibrary{calc,intersections,through,backgrounds}
\renewcommand\citep[1]{\citet*{#1}}

\newcommand{\vett}[1]{\left(\begin{array}{cccc}#1 \end{array}\right)}

\newtheorem{thm}{Theorem}
\newtheorem{remark}{Remark}

\newtheorem{example}{Example}
\newtheorem{prop}{Proposition}
\newtheorem{lem}{Lemma}

\newtheorem{assum}{Assumption}

\newtheorem{problem}{Problem}

\newcommand{\Real}{\mathbb{R}}

\newcommand{\bq}{\mathbf{q}}

\newcommand{\bd}{\mathbf{d}}
\newcommand{\bc}{\mathbf{c}}
\newcommand{\bg}{\mathbf{g}}
\newcommand{\btau}{\boldsymbol{\tau}}
\newcommand{\ba}{\boldsymbol{a}}
\newcommand{\bbold}{\boldsymbol{b}}
\newcommand{\bmu}{\boldsymbol{\mu}}
\newcommand{\balpha}{\boldsymbol{\alpha}}
\newcommand{\bgamma}{\boldsymbol{\gamma}}
\newcommand{\blambda}{\boldsymbol{\lambda}}
\newcommand{\etab}{\boldsymbol{\eta}}
\newcommand{\be}{\boldsymbol{1}}
\newcommand{\bpsi}{\boldsymbol{\psi}}
\newcommand{\bbeta}{\boldsymbol{\beta}}
\newcommand{\bD}{\boldsymbol{D}}
\newcommand{\bC}{\boldsymbol{C}}

\newcommand{\bx}{\boldsymbol{x}}
\newcommand{\fb}{\boldsymbol{f}}
\newcommand{\belll}{\boldsymbol{\ell}}

\begin{document}
%\runninghead{Optimal time-complexity speed planning}
\title{A fast speed planning algorithm for robotic manipulators}
\providecommand{\keywords}[1]{\textbf{\textit{Index terms---}} #1}
\author{Luca Consolini$^1$, Marco Locatelli$^1$, Andrea Minari$^1$, \'Akos Nagy$^2$, Istv\'an Vajk$^2$}

\date{\small $^1$ Dipartimento di Ingegneria e Archittetura,
  University of Parma, Italy\\
$^2$ Department of Automation and Applied Informatics,
Budapest University of Technology and Economics, Hungary\\ \hfill \\
February 2018}
\maketitle

\begin{abstract}
We consider the speed planning problem for a robotic
manipulator.
In particular, we present an algorithm for finding the time-optimal
speed law along an assigned path that satisfies velocity and acceleration constraints
and respects the maximum forces and torques allowed by the actuators.
The addressed optimization problem is a finite dimensional reformulation of the
continuous-time speed optimization problem, obtained by discretizing the
speed profile with $N$ points.
The proposed algorithm has linear complexity with respect to $N$ and to the number of degrees of freedom. Such complexity is the best possible for this problem.
Numerical tests show that the proposed algorithm is significantly faster than
  algorithms already existing in literature.
\end{abstract}

\keywords{time-optimal control, motion planning, robot manipulator}

\maketitle

\section{Introduction}

For robotic manipulators, the motion planning problem
is often decomposed into two subproblems:
path planning and speed planning~\citep{Lavalle2006}.

The first problem consists in finding a path (i.e., the curve
followed by the joints) that joins assigned initial and final
positions. The second problem consists in finding the time-optimal
speed law along the path that satisfies assigned velocity and acceleration constraints
and respects the maximum forces and torques allowed by the actuators.
In this paper we consider only the second problem.
Namely,
given a path $\Gamma$ in the robot configuration space, we want
to find the optimal speed-law that allows following $\Gamma$ while
satisfying assigned kinematic and dynamics constraints. More specifically, we consider
the problem
\begin{equation}
\label{prob:general}
\begin{array}{ll}
\displaystyle\min_{\bq,\btau} & t_{f} \\[6pt]
&\bD(\bq)\ddot{\bq} + \bC(\bq,\dot{\bq} )\dot{\bq} + \belll(\bq) = \btau,\\ [6pt]
&(\dot{\bq}^2,\ddot{\bq},\btau) \in \mathcal{C}(\bq),\\ [6pt]
&\bq \in \Gamma,
\end{array}
\end{equation}
where: $t_{f}$ is the travel time;  $\bq$ is the generalized position; $\btau $ is the generalized force vector;
 $\bD(\bq)$ is the mass matrix; $\bC(\bq,\dot{\bq} )$ is a matrix accounting for centrifugal and
Coriolis effects; $\belll(\bq)$ is an external force term (for instance gravity);
$\mathcal{C}(\bq)$ is a set that represents the kinematic and dynamic limitations of the manipulator.
 %\[
%\begin{array}{ll}
%\dot{q}^2 \le (\dot{q}^{\max})^2, \\ [8pt]
%-\ddot{q}^{\max}  \leq  \ddot{q} \leq  \ddot{q}^{\max},  \\ [8pt]
%\tau^{\max}  \leq  \tau \leq  \tau^{\max}.
%\end{array}
%\]
%where $\dot{q_i}^{\max} > 0$ is the maximal velocity vector, $\ddot{q}^{\max} > 0$ is the maximal
%acceleration vector and $\tau^{\max}$ is the torque constraint for the joints of the $p$-DOF manipulator.\\

\subsection{Related Works}

There are mainly three different families of speed profile generation methods: 
\emph{Numerical Integration},
\emph{Dynamic Programming},
and \emph{Convex Optimization}.

References~\citep{Bobrow1985,pfeiffer1987}  are among the first works
that study Problem~\ref{prob:general} using the \emph{Numerical Integration} approach.
In particular, they  find the time-optimal speed law  \emph{as a function of arc-length} and not as a function of time. 
This choice simplifies the mathematical structure of the resulting problem and has been adopted by 
most of successive works. In~\citep{Bobrow1985,pfeiffer1987} the optimization problem is solved with iterative algorithms.
In particular, reference \citep{Bobrow1985} finds the points in which
the acceleration changes sign using the numerical integration of the second order differential equations 
representing the motions obtained with the maximum and minimum possible accelerations.
Reference~\citep{pfeiffer1987}  is based on geometrical considerations on the feasible set.
However, this approach  has some limitations due to the determination of the \emph{switching points}
that is the main source of failure of this approach (see~\citep{Slotine1989Improving,shiller1992computation}).
For recent results on \emph{Numerical Integration} see~\citep{pham2014general,pham2013kinodynamic,pham2015time}.
For instance~\citep{pham2015time} considers the case of redundant manipulators.

In the \emph{Dynamic Programming} approach the problem is solved with
a finite element approximation of the Hamilton-Jacobi-Bellman equation
(see~\citep{shin1986dynamic,singh1987optimal,Oberherber15}). The main
difficultly with this approach is the high computational time due to
the need of solving a problem with a large number of variables.

The \emph{Convex Optimization} approach is based on the approximation of Problem~\eqref{prob:general} with a finite dimensional optimization problem obtained through spatial discretization. 
Reference~\citep{verscheure09} is one of the early works using this approach.
It shows that Problem~\eqref{prob:general} becomes convex after a change of variables and
that a discretized version of Problem~\eqref{prob:general} is a \emph{Second-Order Cone Programming} (SOCP) problem.
%that can be solved more quickly that a general convex problem~\citep{boyd2004convex}.
This approach has the advantage that the optimization problem can be
tackled with available solvers (e.g., see~\citep{LippBoyd2014,gurobi}).
Moreover, differently from the \emph{Numerical Integration}, this approach allows considering other convex objective functions.
However, the convex programming approach could be inappropriate (see for instance~\citep{pham2014general})
 for online motion planning since the computational time grows
 rapidly (even if still polynomially) with respect to the number of samples in the discretized problem.
Subsequent works, starting from~\citep{verscheure09}, extend the applicability of this approach to different scenarios (see~\citep{debrouwere2013time,csorvasi2017near}) 
and propose algorithms that reduce the computational time (see~\citep{hauser2014fast,Nagy2017}).
To reduce computational time, reference~\citep{hauser2014fast} 
proposes an approach based on \emph{Sequential Linear Programming} (SLP). 
Namely, the algorithm proposed in~\citep{hauser2014fast} sequentially linearizes the objective function
around the current point, while a trust region method ensures the convergence of the process.
%Reference~\citep{verscheure09} shows that, after a change of variables,
%Problem~\ref{prob:general} becomes convex. It also shows that, by spatial discretization,
%Problem~\ref{prob:general} can be approximated by a Second-Order Cone Programming (SOCP) problem.
%Spatial discretization consists in dividing the path in small segments of equal length and approximating the optimal speed-planning problem by a finite dimensional one whose variables are the squared velocities at the beginning of each segment.
Further,~\citep{Nagy2017} shows that, using a suitable discretization method, the time optimal velocity profile can be obtained
by \emph{Linear Programming} (LP) with the benefit of a lower
computation time with respect to convex solvers.

A very recent, and very interesting, paper, closely related to our
work, is~\citep{pham2017new}. In Section~\ref{sec_comparison}, we will shortly
describe the approach proposed there and compare
it with our approach.

Our approach combines the ideas which we previously proposed in two
other works. Namely, in~\citep{minSCL17} we proposed an exact
linear-time forward-backward algorithm for the solution of a velocity
planning problem for a vehicle over a given trajectory under velocity, normal and
tangential acceleration bounds. In~\citep{csorvasi2017near}, a method based on the sequential solution of two-dimensional subproblems
is proposed for the solution of the so-called waiter motion problem. The method is able to return a feasible, though not necessarily optimal, solution.
In the current paper we merge the ideas proposed in the two above mentioned papers in order to derive an approach for the speed planning of robotic manipulators. This
will be proved to return an optimal solution and to have linear time
complexity both with respect to the number of discretization points
and to the number of degrees of freedom of the robotic manipulator.

%Finally, a recent work proposes a method  that computes the
%optimal velocity profile by solving a sequence of small \emph{Linear Programs} (LPs)~\citep{pham2017new}. 
%In particular, this approach is composed of two phases. A \emph{backward phase},
 %in which a sequence of two dimensional LP problems are solved for finding feasible velocity intervals
%(called \emph{Controllable sets}), and a \emph{forward phase}, where a sequence of two dimensional
% LP problems is solved for finding the optimal
%velocity values (\emph{Reachable sets}) into the interval returned by the \emph{backward phase}.
%\citep{pham2017new} proves the correcteness of its approach under the assumption that there are no \emph{zero-inertial}
% points~\citep{Slotine1989Improving} (or no \emph{critical points}~\citep{shiller1992computation} ) while, in case of \emph{zero-inertial points}, shows that the returned solution is sub-optimal and that the gap with the optimal solution converges to zero with a decreasing step size. Moreover, the authors shows that this approach is more fast and robust than the \emph{Numerical Integration}.
 
%Note that the algorithm we proposes (see Section~\ref{sec:solver}) has some connection with the approach
%presented in \citep{pham2017new}. We will discuss this point in Section \ref{sec:conclusion}.

%%

\subsection{Main results}
The purpose of this paper is to provide a speed planning method for robotic manipulators with optimal
time complexity. 
With respect to the existing literature, the new contributions of this work are the following ones.
\begin{itemize}
\item We propose a new algorithm for solving a finite dimensional reformulation of
  Problem~\eqref{prob:general} obtained with $N$ discretization points.
\item We show that if set $\mathcal{C}(\bq)$ in
  Problem~\eqref{prob:general} is defined by linear constraints, then
  the proposed
algorithm has complexity $O(pN)$, where $N$ is the number of discretization
points and $p$ is the number of degrees of freedom. Moreover, such complexity is optimal.
\item By numerical tests, we show that the proposed procedure is significantly faster than
  algorithms already existing in literature.
\end{itemize}
%\item We define a class of problems that contains the problem obtained
 % by discretizing the minimum time speed profile planning
 % (\ref{prob:general}).
%\item We propose an algorithm that solves the problems belonging to
 % this class.

\subsection{Paper Organization}
In Section~\ref{sec:problem},  we present the time-optimal control problem for robotic manipulators in continuous time. In Section~\ref{sec:solver},
we present a class of optimization problems and an exact solution algorithm.
We prove the correctness of the algorithm and compute its time complexity, showing that such complexity is optimal in case of linear constraints.
In Section~\ref{sec:Disc-Problem},
we show that by suitably discretizing the continuous time problem, it is possible to obtain a finite dimensional problem with linear constraints that falls into the class defined in Section~\ref{sec:solver}. 
 Finally, we present an experiment for a 6-DOF industrial
robotic manipulator and we compare the performance of the proposed approach with that of existing solvers
(see~\citep{LippBoyd2014, gurobi, Nagy2017}).

\subsection{Notation}

We denote with  $\Real_{+}$ the set of nonnegative real numbers.
For a vector $\bx \in \Real^n$, $|\bx| \in \Real_{+}^n$ denotes
the component-wise absolute value of $\bx$ and we define the norms $\| \bx\|_{2}:= \sqrt{\sum_{i=1}^{n}| x_{i}| ^{2}}$,
$\| \bx \|_{\infty} := \max\{ |x_{1} |,\dots, | x_{n} | \}$. We also set $\be= [1 \dots1]^{T}$.

For $r \in \mathbb{N}$, we denote by $C^{r}([a,b],\Real^{n})$ the set
of continuous functions from $[a,b] \subset \Real$ to $ \Real^{n}$
that have continuous first $r$ derivatives. For $f \in
C^{1}([a,b],\Real)$, $f^{\prime}$ denotes the derivative and notation
$\dot{f}$ is used if $f$ is a function of time. We set
$\| \fb \|_{\infty} := \sup_{i = 1,\dots,n} \sup \{ |f_{i}(x)| : x
\in[a,b] \}$.
We say that $\fb : [a,b] \to \Real^{n}$ is bounded if there exists $M\in \Real$ such that $\|\fb(x)\|_{\infty} \leq M$.

Consider $h,g: \mathbb{N} \to \Real$. We say that $h(n) = O(g(n))$, if there exists a positive constant $M$ such that, for all sufficiently large values of $n$, $|h(n)| \le M|g(n)|$.

\section{Problem formulation}
\label{sec:problem}

Let $\mathcal{Q}$ be a smooth manifold of dimension $p$ that represents the
configuration space of a robotic manipulator with $p$-degrees of
freedom ($p$-DOF).
Let $\Gamma : [0,1] \to \mathcal{Q}$ be a smooth curve whose image set
$\operatorname{Im} \Gamma$ represents the assigned path to be followed
by the manipulator.
We assume that there exist two open sets
$U \supset \operatorname{Im} \Gamma$, $V \subset\Real^{p}$ and an invertible
and smooth function $\phi : U \to V$.
Function $\phi$ is a local chart that allows representing each
configuration $q \in U$ with coordinate vector $\phi(q) \in \Real^p$.

%Due to the function $\phi$, for the rest
%of the paper we can consider $\Real^{p}$ as the coordinate space.

%Since $\mathcal{Q}$ is a manifold then we can define a local chart.
%A chart of a set $\mathcal{Q}$ is a  pair $(U,\phi)$ with $U$ a subset of $\mathcal{Q}$ and $\phi:U \to \phi(U) \subset \Real^{p}$ a bijection from a open set $U$ to an open set $\phi(U)$ in $\Real^{p}$.
The coordinate vector $\bq$ of a trajectory in $U$ satisfies the dynamic equation
\begin{equation}
\label{eq:manip}
\bD(\bq)\ddot{\bq} + \bC(\bq,\dot{\bq} )\dot{\bq} + \belll(\bq) = \btau,
\end{equation}
where $\bq \in \Real^{p}$ is the generalized position vector, $\btau \in \Real^{p}$ is the generalized force vector, $\bD(\bq)$ is the mass matrix, $\bC(\bq,\dot{\bq})$ is the matrix accounting for centrifugal and
Coriolis effects (assumed to be linear in $\dot{\bq}$) and $\belll(\bq)$ is the vector accounting for joints position dependent forces, including gravity.
Note that we do  not consider Coulomb friction forces.

Let $\bgamma \in C^2([0,s_{f}],\Real^{p}) $ be a function such that
$\phi(\bgamma[0,s_f])=\operatorname{Im} \Gamma$ and ($\forall \lambda \in [0,s_f]$)
$\lVert \bgamma^\prime(\lambda) \rVert =1$.
The image set  $\bgamma([0,s_f])$ represents the
coordinates of the elements of reference path $\Gamma$.
In particular, $\bgamma(0)$ and $\bgamma(s_f)$ are the coordinates of
the initial and final configurations. Define $t_{f}$ as the time when the robot reaches the end of the path. Let
 $\lambda : [0, t_f] \rightarrow [0, s_f]$ be a differentiable monotone increasing function that represents
 the position of the robot as a function of time and let   $ v : [0, s_f] \rightarrow [0, +\infty]$ be such that
  $\left( \forall t \in [0,t_f]\right) \dot{\lambda}(t) = v(\lambda(t))$. Namely, $v(s)$ is the velocity of the
 robot at position $s$. We impose ($\forall s \in [0,s_{f}]$) $v(s) \ge 0$.
For  any $t \in [0,t_f]$, using the chain rule, we obtain
\begin{equation}
\label{eq:rep}
\begin{array}{ll}
\bq(t) =& \bgamma(\lambda(t)),\\[8pt]
\dot{\bq}(t) = & \bgamma^{\prime}(\lambda(t))v(\lambda(t)),\\[8pt]
\ddot{\bq}(t) = & \bgamma^{\prime}(\lambda(t))v^\prime(\lambda(t))v(\lambda(t)) + \bgamma^{\prime\prime}(\lambda(t))v(\lambda(t))^2.
\end{array}
\end{equation}

Substituting (\ref{eq:rep}) into the dynamic equations (\ref{eq:manip}) and setting $s = \lambda(t)$, we rewrite the dynamic equation (\ref{eq:manip}) as follows:\\
\begin{equation}
\label{eq:dynamic}
\bd(s)v^{\prime}(s)v(s) + \bc(s)v(s)^2 + \bg(s) = \btau(s) ,
\end{equation}
where the parameters in (\ref{eq:dynamic}) are defined as
\begin{equation}
\label{eq:dynamic_parameters}
\begin{array}{l}
\bd(s) = \bD(\bgamma(s))\bgamma^{\prime}(s),\\ [8pt]
\bc(s) =  \bD(\bgamma(s))\bgamma^{\prime\prime}(s)  + \bC(\bgamma(s),\bgamma^{\prime}(s))\bgamma^{\prime}(s),  \\ [8pt]
\bg(s) = \belll(\bgamma(s)).
\end{array}
\end{equation}
The objective function is given by the overall travel time $t_f$
defined as
\begin{equation}
\label{eq:objective}
\displaystyle t_f = \int_0^{t_f}1\,dt = \int_{0}^{s_f} v(s)^{-1}\, ds.
\end{equation}

Let $\bmu, \bpsi, \balpha : \left[ 0, s_f \right] \rightarrow \Real^{p}_{+}$ be assigned bounded functions and
consider the following minimum time problem:
\begin{problem}\label{prob:1}
\begin{align}
\displaystyle\min_{v \in C^{1},\btau\in C^{0}} &  \displaystyle\int_0^{s_f} v(s)^{-1} \, ds, \label{obj:v}\\
\mbox{\small\textrm{subject to}} &  \ (\forall s \in [0,s_{f}]) \nonumber\\
& \bd(s)v^{\prime}(s)v(s) + \bc(s)v(s)^2 + \bg(s) = \btau(s), \label{con:dynamic}\\
& \bgamma^{\prime}(s)v(s) = \dot{\bq}(s),\label{con:kinematic1} \\
&   \bgamma^{\prime}(s)v^\prime(s)v(s) + \bgamma^{\prime\prime}v(s)^{2} = \ddot{\bq}(s),\label{con:kinematic2} \\
& \lvert  \btau(s) \rvert  \le \bmu(s), \label{con:force_bound}\\
&  \lvert   \dot{\bq}(s) \rvert \le \bpsi(s),\label{con:vel_bound} \\
& \lvert \ddot{\bq}(s) \rvert \le \balpha (s), \label{con:acc_bound}\\
&v(s) \ge 0, \label{con:positive-velocity} \\
& v(0) = 0, \, v(s_f) =  0,  \label{con:interpolation}
\end{align}
\end{problem}
where (\ref{con:dynamic}) represents the robot dynamics, (\ref{con:kinematic1})-(\ref{con:kinematic2}) represent the relation between  the path $\bgamma$ and the generalized position
$q$ shown in~(\ref{eq:rep}), (\ref{con:force_bound}) represents the
bounds on generalized forces,  (\ref{con:vel_bound}) and (\ref{con:acc_bound}) represent the
bounds on joints velocity and acceleration. Constraints~(\ref{con:interpolation}) specify the
 interpolation conditions at the beginning and at the end of the path.
 
The following assumption is a basic requirement for fulfilling constraint~(\ref{con:vel_bound}).
\begin{assum}\label{ass:psi}
We assume that $\bpsi$ is a positive continuous function, i.e.,
 ($\forall s \in [0,s_{f}]$)  $\psi_{i}(s) > 0$ with $i=1,\dots,p$.
\end{assum}

Next assumption requires that the maximum allowed generalized
forces are able to counteract external forces (such as gravity) when the manipulator is fixed at
each point of $\Gamma$.

\begin{assum}\label{ass:mu}
We assume that $\exists \varepsilon \in \Real$, $\varepsilon >0$ such that
$(\forall s \in [0,s_{f}])$ $\bmu(s) - |\bg(s)| > \varepsilon\be$.
\end{assum}
In fact for $v = 0$ condition (\ref{con:force_bound}) reduces to ($\forall s \in [0,s_{f}]$) $|\bg(s)| \le \bmu(s)$.

Problem~\ref{prob:1} is nonconvex, but it becomes convex after a
simple change of variables (as previously noted in~\citep{verscheure09}).
Indeed,  $(\forall s \in [0,s_f])$  set
\begin{equation}
\label{eq:change}
a(s) = v^{\prime}(s)v(s),  \quad b(s) = v(s)^2,
\end{equation}
and note that
\begin{equation}
\label{eq:relation}
b^\prime(s) = 2a(s).
\end{equation}
%where (\ref{eq:relation}) defines the relations between the functions
%$b$ and $a$.
Then, Problem~\ref{prob:1} becomes:
 \begin{problem}\label{prob:2}
\begin{align}
\displaystyle\min_{a,\btau \in C^{0}, b \in C^{1}} &  \displaystyle\int_0^{s_f} b(s)^{-1/2} \, ds, \label{con:obj-2}\\
\mbox{\small\textrm{subject to}} &  \ (\forall s \in [0,s_{f}]) \nonumber\\
& \bd(s)a(s) + \bc(s) b(s) + \bg(s) = \btau(s), \label{con:dynamic-2} \\
&\bgamma^{\prime}(s)a(s)+ \bgamma^{\prime\prime}(s)b(s) = \ddot{\bq}(s), \\
& b^\prime(s) = 2a(s), \label{con:der}\\
& \lvert  \btau(s) \rvert  \le \bmu(s), \label{con:force-2}\\
&  0 \le  \bgamma^{\prime}(s)^{2} b(s) \le \bpsi(s)^{2}, \label{con:vel-2}\\
& \lvert \ddot{\bq}(s) \rvert \le \balpha (s), \label{con:acc-2}\\
& b(0) = 0, \, b(s_f) =  0,  \label{con:interp-2}
\end{align}
\end{problem}
where the squares of the two vectors $\bgamma^{\prime}(s)$ and $\bpsi(s)$ in (\ref{con:vel-2}) are to be intended component-wise.
Problem~\ref{prob:2} is convex since the objective function (\ref{con:obj-2}) is convex and the constraints~(\ref{con:dynamic-2})-(\ref{con:interp-2}) are linear.

The following proposition (that will be proved in the appendix) shows
that Problem~\ref{prob:2} admits a solution.

\begin{prop}\label{prop-solution-existence}
Problem \ref{prob:2}  admits an optimal solution $b^{*}$, and moreover,
\[
\displaystyle\int_0^{s_f} b^{*}(s)^{-1/2} \, ds  \le U < \infty,
\]
where $U$ is a constant depending on problem data.
\end{prop}

%%%%%%%%%
We do not directly solve  Problem~\ref{prob:2}, but find an
approximated solution based on a finite dimensional approximation. Namely, consider the following problem, obtained by uniformly sampling the interval $\left[ 0, s_{f} \right]$  in
$n$ points  $\ s_{1},\ s_{2}, \dots, s_{n}$  from $s_{1} = 0$ to $s_{n} = s_f$ :
\begin{problem}\label{prob:disc}
\begin{align}
\displaystyle\min_{\btau ,\ba,\bbold }& \ 2h\displaystyle\sum_{i=1}^{n-1} \left( \frac{1}{b_i^{1/2} + b_{i+1}^{1/2}} \right). \label{con:obj-disc}\\
\mbox{\small\textrm{subject to}} &  \ (i=1,\dots,n-1)\\
& \bd_i a_i + \bc_i b_i + \bg_i = \btau_i, \label{con:dynamic-disc} \\
&  \bgamma^{\prime}_i a_i+ \bgamma^{\prime\prime}_i b_i = \ddot{\bq}_{i},  \label{con:acc-disc} \\
&b_{i+1} - b_{i} = 2 a_i h ,\label{con:approx-der} \\
& \lvert \btau_i  \rvert \le \bmu_i, \label{con:force-disc}\\
& \lvert \ddot{\bq}_{i}  \rvert \le \balpha_i \\
&  0 \le  \left[\bgamma_{i}^{\prime}\right]^2 b_i \le \boldsymbol{\psi}_i^{2},  \label{con:vel-disc}\\
& b_1 = 0, \, b_n = 0. \label{con:inter-disc}\\
&\boldsymbol{b} \in \Real^{n} \ \ba \in \Real^{n-1},\btau_{i}\in\Real^{p}, \label{con:set.disc}
\end{align}
\end{problem}
where $h=\frac{s_n}{n-1}$,
  $\balpha_{i} = \alpha(s_{i})$,
  $\bpsi_{i} = \psi(s_{i})$,
  $\bmu_{i} = \mu(s_{i})$,
  $\bgamma^{\prime}_i = \bgamma^{\prime}(s_i) $,
  $\bgamma^{\prime\prime}_i= \bgamma^{\prime\prime}(s_i)$,
  $\bd_i = D(\bgamma_i)\bgamma^{\prime}_i$,
  $\bc_i = D(\bgamma_i)\bgamma^{\prime\prime}_i  + C(\bgamma_i,\bgamma^{\prime}
_i)\bgamma^{\prime}_i$, and $\bg_i = g(\bgamma_i)$, with $i =  1, \dots, n $.

Thank to constraints (\ref{con:dynamic-disc})-(\ref{con:approx-der}), it is possible to eliminate variables $\btau_{i}$ and $a_{i}$ and use only $b_{i}$, with $i = 1,\dots,n,$ as decision variables.
The feasible set of Problem~\ref{prob:disc} is a non-empty set since $\boldsymbol{b} = 0$
is a feasible solution (in fact, it also has a nonempty interior).

Since Problem~\ref{prob:disc} is convex, we can easily find a solution with an interior point method (see~\citep{verscheure09}). %Moreover, reference \citep{Nagy2017}
%shows that, because the objective function (\ref{con:obj-disc}) is strictly monotone increasing with respect to $b_{i}$, with
%$i=0,\dots,n$,  it can be substituted with $\max \sum_{i=0}^{n} b_{i}$ (i.e., maximize the velocity along the path) obtaining an LP
%problem.

After solving Problem~\ref{prob:disc}, it is possible to find an approximated solution  of
Problem \ref{prob:2}. Indeed, by quadratic interpolation, we associate
to a vector $\bbold \in  \Real^n$, solution of Problem~\ref{prob:disc}, a continuously differentiable function
$\mathcal{I}_{\bbold}:[0, s_f] \to \Real$ such that the following relations hold
for $i=1,\ldots,n-1,$
\begin{equation}
\label{eq:interpolfun}
\begin{array}{l}
\mathcal{I}_{\bbold}(0)=b_1,\ \ \mathcal{I}_{\bbold}(s_f)=b_n,\\ [8pt]
\mathcal{I}_{\bbold}\left((i-1/2)h\right)=\frac{b_i+b_{i+1}}{2}, \ i=1,\ldots,n-1, \\  [8pt]
\mathcal{I}_{\bbold}^\prime\left((i-1/2)h\right)=\frac{b_{i+1}-b_{i}}{ h}, \ i=1,\ldots,n-1.
\end{array}
\end{equation}
Namely, $\mathcal{I}_{\bbold}$ interpolates $b_1$ and $b_n$ at 0 and $s_f$,
respectively, and the average values of consecutive entries of $\bbold$ at
the midpoint of the discretization intervals. Moreover, the derivative
of $\mathcal{I}_{\bbold}$ at the midpoints of the discretization intervals
corresponds to the finite differences of $\bbold$.

We define the class of quadratic splines $\mathcal{P}$ as the subset of
$\mathcal{C}^1([0,s_f],\Real)$, such that, for $i=1,\ldots,n-1$,
$p|_{[h(i-\frac{1}{2}),h(i+\frac{1}{2})]}$ is a quadratic polynomial.
% such that
%for all $i=1,\ldots, n-1$, $p|_{[h(i-\frac{1}{2}),h(i+\frac{1}{2})]}$,
%$p|_{[0,\frac{h}{2}]}$ and $p|_{[h(n-1/2),hn]}$ are polynomials of degree $2$.
For $i=2,\ldots,n-1$,
let $x_i, y_i, z_i \in \Real$ be such that:
\[
p|_{[h(i-\frac{1}{2}),h(i+\frac{1}{2})]}(t)=x_i+y_i (t-hi)+ z_i
(t-hi)^2,
\]
and let $x_1,y_1,z_1,x_n,y_n,z_n$ be such that
\[
p|_{[0,h/2]}(t)=x_1+y_1 t+ z_1
t^2
\]
and
\[p|_{[h(n-1/2),hn]}(t)=x_n+y_n (t-hn)+ z_n
(t-hn)^2.
\]
For ease of notation, set, for $i=1,\ldots,n-1$,
\[
\begin{array}{l}
b_{i-1/2}=\displaystyle\frac{b_{i-1}+b_{i}}{2},\ \ \ b_{i+1/2}=\frac{b_{i+1}+b_i}{2} \\ [8pt]
\delta_{i+1/2}=\displaystyle\frac{b_{i+1}-b_i}{h},\ \ \ \delta_{i-1/2}=\frac{b_{i}-b_{i-1}}{h}.
\end{array}
\]

The following proposition, whose proof is presented in
appendix, defines the interpolating quadratic spline fulfilling~(\ref{eq:interpolfun}).
\begin{prop}
\label{prop_interp_cond}
For any $b \in \Real^n$, there exists a unique element $p\in \mathcal{P}$ such that the following
interpolation conditions hold

\begin{equation}
\label{eqn_interp_cond}
\begin{array}{ll}
p(h(i+\frac{1}{2}))=b_{i+1/2},& i=1,\ldots,n-1,\\ [8pt]
p'(h(i+\frac{1}{2}))=\delta_{i+1/2},& i=1,\ldots,n-1,\\ [8pt]
p(0)=b_1, \ \ p(s_f)=b_n\,.
\end{array}
\end{equation}
Note that $p$ is continuously differentiable and that (\ref{eq:interpolfun}) holds.
Such element will be denoted by  $\mathcal{I}_{\bbold}$.
\end{prop}

%%%%%%%%%%%
%To this end, define $b$ as a quadratic spline
%obtained by interpolating values $b_i$ as follows.
%%%the class of quadratic spline subset of $C^{1}(\left[ 0 ,s_{f} \right], \Real)$, the function $\tilde{b} \in\mathcal{P}$ is defined
%
%\begin{equation}
%\label{eq:piecewise-polynomial}
%b(s) = \begin{cases}
%p_{1}(s) = x_{1} + y_{1}s + z_{1}s^{2}, & s \in \left[s_{0},s_{1}\right], \\
%p_{2}(s) = x_{2} + y_{2}s + z_{2}s^{2}, & s \in \left[s_{1},s_{2}\right], \\
%\dots \\
%p_{n}(s) = x_{n} + y_{n}s + z_{n}s^{2}, & s \in \left[s_{n-1},s_{n}\right], \\
%\end{cases}
%\end{equation}
%where $x_{i}$, $y_{i}$ ,$z_{i} \in \Real$, for $i=1,\dots,n$ are such that
%the following interpolation conditions hold
%\begin{equation}
%\label{eq:interpolation-conditions}
%\begin{array}{ll}
%p_{i}(s_{i-1}) = b_{i-1}, \ p_{i}(s_{i})  = b_{i}, & i = 1,\dots,n ,\\ [8pt]
%p^{\prime}_{i}(s_{i}) = p^{\prime}_{i+1}(s_{i}) & i = 1, \dots, n-1. \\ [8pt]
%p_{1}^{\prime}(s_{0}) = 0, \ p_{n}^{\prime}(s_{n}) = 0.
%\end{array}
%\end{equation}
%

Note that by Proposition~\ref{prop_interp_cond}, there exists a
unique function $b = \mathcal{I}_{\bbold}$ that interpolates the solution of Problem~\ref{prob:disc}.
Then $a$ and $\tau$ are computed from $b$ by using relations~(\ref{eq:relation})
 and~(\ref{con:dynamic-2}), namely we~set $ (\forall s \in \left[ 0, s_{f} \right])$
\[
\begin{array}{l}
a(s) = \displaystyle\frac{b^{\prime}(s)}{2}, \\ [8pt]
\btau(s) = \bd(s)a(s) + \bc(s) b(s) + \bg(s).
\end{array}
\]
Functions $b$, $a$ and $\tau$ are approximate solutions of Problem~\ref{prob:2}.
 Indeed,  (\ref{con:obj-disc})~and~(\ref{con:approx-der}) are approximations of
(\ref{con:obj-2})~and~(\ref{con:der}), moreover, functions  $b$, $a$ and $\tau$ satisfy, by construction, constraints
(\ref{con:dynamic-2})-(\ref{con:interp-2}) for $s \in \left\{ s_{1}, \dots,s_{n}\right\}$ and by continuity,
 (\ref{con:dynamic-2})-(\ref{con:interp-2}) are also approximately satisfied for $s\in\left[ 0, s_{f} \right]$.
By increasing the number of samples $n$, the solutions of
Problem~\ref{prob:disc} become better approximations
 of the solutions of Problem~\ref{prob:2}.
It is reasonable to suppose that, as $n$ approaches to $+\infty$, the solutions of  Problem~\ref{prob:disc}
converge to the solutions
of Problem~\ref{prob:2}. Anyway this convergence property is not
proved in this paper being outside its scope. It can be proved on the
lines of~\citep{Consolini2017converges}, that presents a convergence
result for a related speed planning
problem for an autonomous vehicle. %but a convergence result goes beyond the focus of this paper.

%%%%%%%%%%

\section{Solution algorithms and complexity issues for the generalized problem}
\label{sec:solver}

In this section we present an optimal time complexity algorithm that
solves a specific class of optimization problems. In the
subsequent section we will show that Problem~\ref{prob:2} can be
approximated by a finite dimensional problem that belongs to this
class.
\subsection{Exact algorithm for the solution of some special structured problems}
The problems under consideration have the form
\begin{equation}
\label{eq:probl}
\begin{array}{lll}
\min & g(v_1,\ldots,v_n) & \\ [8pt]
& v_i\leq f_j^{i}(v_{i+1}) & i=1,\ldots,n-1,\\[8pt]
&&j=1,\ldots,r_i \\ [8pt]
& v_{i+1}\leq b_k^{i}(v_{i}) & i=1,\ldots,n-1,\\ [8pt]
&&k=1,\ldots,t_i, \\ [8pt]
 & 0\leq v_i\leq u_i & i=1,\ldots,n,
\end{array}
\end{equation}
where we make the following assumptions.
\begin{assum}
\label{ass:1}
We assume:
\begin{itemize}
\item $g$ monotonic non increasing;
\item $f_j^i$, concave, increasing and $f_j^i(0)> 0$,\\  $i=1,\ldots,n-1,\ \ j=1,\ldots,r_i$;
\item $b_k^i$, concave, increasing and $b_k^i(0)> 0$,\\  $i=1,\ldots,n-1,\ \ k=1,\ldots,t_i$.
\end{itemize}
\end{assum}
The constraints in (\ref{eq:probl}) can be rewritten in compact form as follows:
$$
\begin{array}{ll}
v_i\leq \min\{B_i(v_{i+1}),u_i\} & i=1,\ldots,n-1, \\ [8pt]
v_{i+1}\leq \min\{F_i(v_{i}), u_{i+1}\} & i=1,\ldots,n-1,
\end{array}
$$
where:
$$
\begin{array}{l}
B_i(v_{i+1})=\min_{j=1,\ldots,r_i} f_j^{i}(v_{i+1}), \\ [8pt]
F_i(v_i)=\min_{k=1,\ldots,t_i} b_k^{i}(v_{i}).
\end{array}
$$
Note that $F_i$ and $B_i$ are both concave and increasing over $\mathbb{R}_+$, since they are the minimum of a finite number of functions with the same properties.
We prove that the same holds for $F_i\circ B_i$.
\begin{prop}
$F_i\circ B_i$ is increasing and concave over $\mathbb{R}_+$.
\end{prop}
\begin{proof}
The fact that $F_i\circ B_i$ is increasing follows immediately from the increasingness of $F_i$ and $B_i$.
For what concerns concavity, $\forall x,y\geq 0,\ \lambda\in [0,1]$:
$$
\begin{array}{l}
F_i\circ B_i(\lambda x+(1-\lambda)y)
= F_i(B_i(\lambda x+(1-\lambda)y))\\ [8pt]
\geq F_i(\lambda B_i(x)+(1-\lambda)B_i(y))
\geq \lambda F_i \circ B_i(x)+(1-\lambda)F_i \circ B_i(y),
\end{array}
$$
where the first inequality is a consequence of the concavity of $B_i$ and the fact that $F_i$ is increasing, while the second inequality
comes from concavity of $F_i$.
\end{proof}
It immediately follows that:
\begin{equation}
\label{eq:fixed}
F_i\circ B_i(x)-x\ \ \mbox{concave},\ \ \  F_i\circ B_i(0)> 0.
\end{equation}
Then, there exists at most one point $\bar{v}_{i+1}>0$ such that $F_i\circ B_i(\bar{v}_{i+1})-\bar{v}_{i+1}=0$.
Similarly, there exists
at most one point $\bar{v}_{i}>0$ such that $B_i\circ F_i(\bar{v}_{i})-\bar{v}_{i}=0$.
Note that $\bar{v}_i, \bar{v}_{i+1}$ are the positive fixed points of $B_i\circ F_i$ and $F_i\circ B_i$, respectively. Alternatively, $(\bar{v}_i, \bar{v}_{i+1})$
is also the optimal solution of the following two-dimensional convex problem:
$$
\begin{array}{lll}
\max & v_i+v_{i+1} & \\ [8pt]
& v_i\leq f_j^{i}(v_{i+1}) & j=1,\ldots,r_i \\ [8pt]
& v_{i+1}\leq b_k^{i}(v_{i}) & k=1,\ldots,t_i.
\end{array}
$$
The following result holds.
\begin{prop}
\label{prop:2}
Under Assumption \ref{ass:1}, the optimal solution of (\ref{eq:probl}) is the component-wise maximum
of its feasible region, i.e., if we denote by $X$ the feasible region, it is the point
${\bf v}^*\in X$ such that ${\bf v}\leq {\bf v}^*$ for all ${\bf v}\in X$.
\end{prop}
\begin{proof}
See \citep{minSCL17,Nagy2017}.
\end{proof}
We consider Algorithm \ref{Alg:FindProfile3} for the solution of problem (\ref{eq:probl}).
\begin{algorithm}
\caption{\label{Alg:FindProfile3} Forward-Backward algorithm for the solution of the problem.}
\SetKwInOut{Data}{Data}
\Data{${\bf u}\in \mathbb{R}_+^n$;}
%\AlgResult{${\bf v}^*$}
Set $\bar{{\bf u}}={\bf u}$\;
\tcc{ {\em Forward phase}}\;
\ForEach{$i\in \{1,\ldots,n-1\}$}
{
Compute the nonnegative fixed points $\bar{v}_i$ and $\bar{v}_{i+1}$ for $B_i\circ F_i$ and $F_i\circ B_i$, respectively
(if they do not exist, set $\bar{v}_i=+\infty$, $\bar{v}_{i+1}=+\infty$)\;
Set $\bar{u}_i=\min\{\bar{u}_i, B_i(\bar{u}_{i+1}), \bar{v}_i\}$\;
Set $\bar{u}_{i+1}=\min\{\bar{u}_{i+1}, F_i(\bar{u}_{i}), \bar{v}_{i+1}\}$\;
}
\tcc{ {\em Backward phase}}\;
\ForEach{$i \in \{n-1,\ldots,1\}$}
{
Set $\bar{u}_i=\min\{B_i(\bar{u}_{i+1}),\bar{u}_i\}$\;
}
\end{algorithm}
The algorithm is correct, as stated in the following proposition.
\begin{prop}
Algorithm \ref{Alg:FindProfile3} returns the optimal solution ${\bf v}^*$ of problem (\ref{eq:probl}).
\end{prop}
\begin{proof}
We first remark that at each iteration $\bar{{\bf u}}\geq {\bf v}^*$ holds.
If a fixed point $\bar{v}_{i+1}$ for $F_i\circ B_i$ exists, then
after the backward propagation, $\bar{u}_{i+1}\leq \bar{v}_{i+1}$, and $\bar{u}_i=\min\{\bar{u}^{old}_i,B_i(\bar{u}_{i+1})\}$, where
$\bar{u}_i^{old}$ denotes the upper bound for $v_i$ after the forward phase.
We show that
\begin{equation}
\label{eq:condit}
F_i(\bar{u}_i)=\min\{F_i(\bar{u}^{old}_i), F_i\circ B_i(\bar{u}_{i+1})\}\geq \bar{u}_{i+1}.
\end{equation}
If this is true for all $i$, then the the point $\bar{{\bf u}}$ at the end of the backward phase is a feasible solution of (\ref{eq:probl}).
Indeed, by definition of $\bar{u}_i$ in the backward phase, $\forall i$
$$
\bar{u}_i\leq B_i(\bar{u}_{i+1}),
$$
while by (\ref{eq:condit}), $\forall i$
$$
\bar{u}_{i+1}\leq F_i(\bar{u}_{i}),
$$
so that $\bar{{\bf u}}$ is feasible for (\ref{eq:probl}).
Since $\bar{{\bf u}}\geq {\bf v}^*$ holds and $g$ is monotone non increasing, we have that $\bar{{\bf u}}$ is the optimal solution of (\ref{eq:probl}).
We only need to prove that (\ref{eq:condit}) is true.
Note that $\bar{u}^{old}_i$ is the result of the first forward propagation, so that
$F_i(\bar{u}^{old}_i)\geq \bar{u}^{old}_{i+1}\geq \bar{u}_{i+1}$. Thus, we need to prove that $ F_i\circ B_i(\bar{u}_{i+1})\geq \bar{u}_{i+1}$ with
$\bar{u}_{i+1}\leq \bar{v}_{i+1}$.
In view of (\ref{eq:fixed}), if the fixed point $\bar{v}_{i+1}$ for $F_i\circ B_i$ exists, it is the unique nonnegative root of
$F_i\circ B_i(x)-x$ and
$$
F_i\circ B_i(x)-x>0\ \ \ \forall x\leq \bar{v}_{i+1},
$$
from which the result is proved. Otherwise, if no fixed point exists,
$$
F_i\circ B_i(x)-x>0\ \ \ \forall x\in \mathbb{R}_+,
$$
form which the result is still proved.
\end{proof}
Under a given condition, Algorithm \ref{Alg:FindProfile3} can be further simplified.
\begin{remark}
If $F_i$ and $B_i$, $i=1,\ldots,n$, fulfill the so called {\em superiority condition}, i.e.
$$
F_i(x), B_i(x)> x \ \ \ \forall x\in \mathbb{R}_+,
$$
then $\bar{v}_i,\bar{v}_{i+1}=+\infty$ and the forward phase can be reduced to
$$
\bar{u}_{i+1}=\min\{\bar{u}_{i+1}, F_i(\bar{u}_{i})\}.
$$
\end{remark}
\subsection{Solving the subproblems in the forward phase and complexity results}
We first remark that
$$
\begin{array}{l}
\bar{u}_i=\min\{\bar{u}_i, B_i(\bar{u}_{i+1}), \bar{v}_i\}, \\ [8pt]
\bar{u}_{i+1}=\min\{\bar{u}_{i+1}, F_i(\bar{u}_{i}), \bar{v}_{i+1}\},
\end{array}
$$
defined in the forward phase of Algorithm \ref{Alg:FindProfile3} are the solution of the two-dimensional convex optimization problem
\begin{equation}
\label{eq:convex sub}
\begin{array}{lll}
\max &\ v_i+v_{i+1} & \\ [8pt]
& v_i\leq f_j^{i}(v_{i+1})\quad  j=1,\ldots,r_i \\ [8pt]
& v_{i+1}\leq b_k^{i}(v_{i})\quad  k=1,\ldots,t_i \\ [8pt]
& 0\leq v_i\leq f_{r_i+1}^i(v_{i+1})\equiv \bar{u}_i & \\ [8pt]
& 0\leq v_{i+1}\leq b_{t_i+1}^i(v_{i})\equiv \bar{u}_{i+1}. &
\end{array}
\end{equation}
Alternatively, $\bar{u}_i, \bar{u}_{i+1}$ can also be detected as fixed points of $B_i^{\bar{{\bf u}}}\circ F_i^{\bar{{\bf u}}}$ and $F_i^{\bar{{\bf u}}}\circ B_i^{\bar{{\bf u}}}$, respectively, where
$$
\begin{array}{l}
F_i^{\bar{{\bf u}}}(x)=\min\{\bar{u}_i,F_i(x)\},\\[8pt]
B_i^{\bar{{\bf u}}}(x)=\min\{\bar{u}_{i+1},B_i(x)\}.
\end{array}
$$
Although any convex optimization or any fixed point solver could be exploited for detecting these values, we propose the simple Algorithm \ref{Alg:convex}, which turns out to be quite effective in practice.
We denote by
$$
[F_i]^{-1}(x)=\max_{j=1,\ldots,r_i} \{[f_j^i]^{-1}(x)\}.
$$
Note that $f_j^i$ increasing and concave implies that $[f_j^i]^{-1}$ is increasing and convex and, consequently, $[F_i]^{-1}$ is increasing and convex .
\begin{algorithm}
\caption{\label{Alg:convex} Algorithm for the computation of the new values $\bar{u}_i$ and $\bar{u}_{i+1}$.}
%\AlgData{${\bf u}\in \mathbb{R}_+^n$, $s,t\in \{1,\ldots,n\}$ $s\leq t$;}
%\AlgResult{${\bf v}^*$}
%\tcc{                   {\em Computation of the vertices in the upper part}}
Let $\bar{x}_1=\bar{u}_i$\;
Set $\bar{y}=-\infty$, $\bar{z}=+\infty$, $h=1$\;
\While{$\bar{y}< \bar{z}$}{
Set $\bar{y}=B_i^{\bar{{\bf u}}}(\bar{x}_h)$ and $\bar{k}\in \{1,\ldots,t_i+1\}\ :\ b_{\bar{k}}^i(\bar{x}_h)=\bar{y}$\;
Set $\bar{z}=[F_i]^{-1}(\bar{x}_h)$ and $\bar{j}\in \{1,\ldots,r_i\}\ :\ [f_{\bar{j}}^i]^{-1}(\bar{x}_h)=\bar{z}$\;
Let $\bar{x}_{h+1}$ be the solution of the one-dimensional equation $f_{\bar{j}}^i(b_{\bar{k}}^i(x))=x$\;
Set $h=h+1$\;
}
\Return{$(\bar{u}_i,\bar{u}_{i+1})=(\bar{x}_h,\bar{y})$}\;
\end{algorithm}
We illustrate how the algorithm works through an example.
\begin{example}
Let
\[
\begin{array}{ll}
b_1^i(x)=\frac{3}{2}x+2, & b_2^i(x)=x+3,\\ [8pt]
b_3^i(x)=\frac{1}{2}x+5, & b_4^i(x)=8.
\end{array}
\]
and
\[
\begin{array}{ll}
[f_1^i]^{-1}(x)=x-1,&  [f_2^i]^{-1}(x)=\frac{9}{2}x-8,\\  [8pt]
[f_3^i]^{-1}(x)=5x-10.
\end{array}
\]
Moreover, let $u_i=u_{i+1}=8$. Then, we initially set $\bar{x}_1=8$.
\newline
In the first iteration we have
\begin{align*}
\bar{y} &= B_i^{\bar{{\bf u}}}(\bar{x}_1)=\min_{k=1,\ldots,4}\{b_k^i(\bar{x}_1)\}=\min\{14,11,9,8\}=8,
\end{align*}
with $\bar{k}=4$, and
\begin{align*}
\bar{z} &=[F_i]^{-1}(\bar{x}_1)=\max_{j=1,\ldots,3}\{[f_j^i]^{-1}(\bar{x}_1)\}=\max\{7,28,30\}=30,
\end{align*}
with $\bar{j}=3$. Then, $\bar{x}_2$ is the solution of the equation $8=5x-10$, i.e., $\bar{x}_2=\frac{18}{5}$
(See also Figure \ref{fig:h1}).
\begin{figure}[!h]
\centering
\includegraphics[width=0.6\columnwidth]{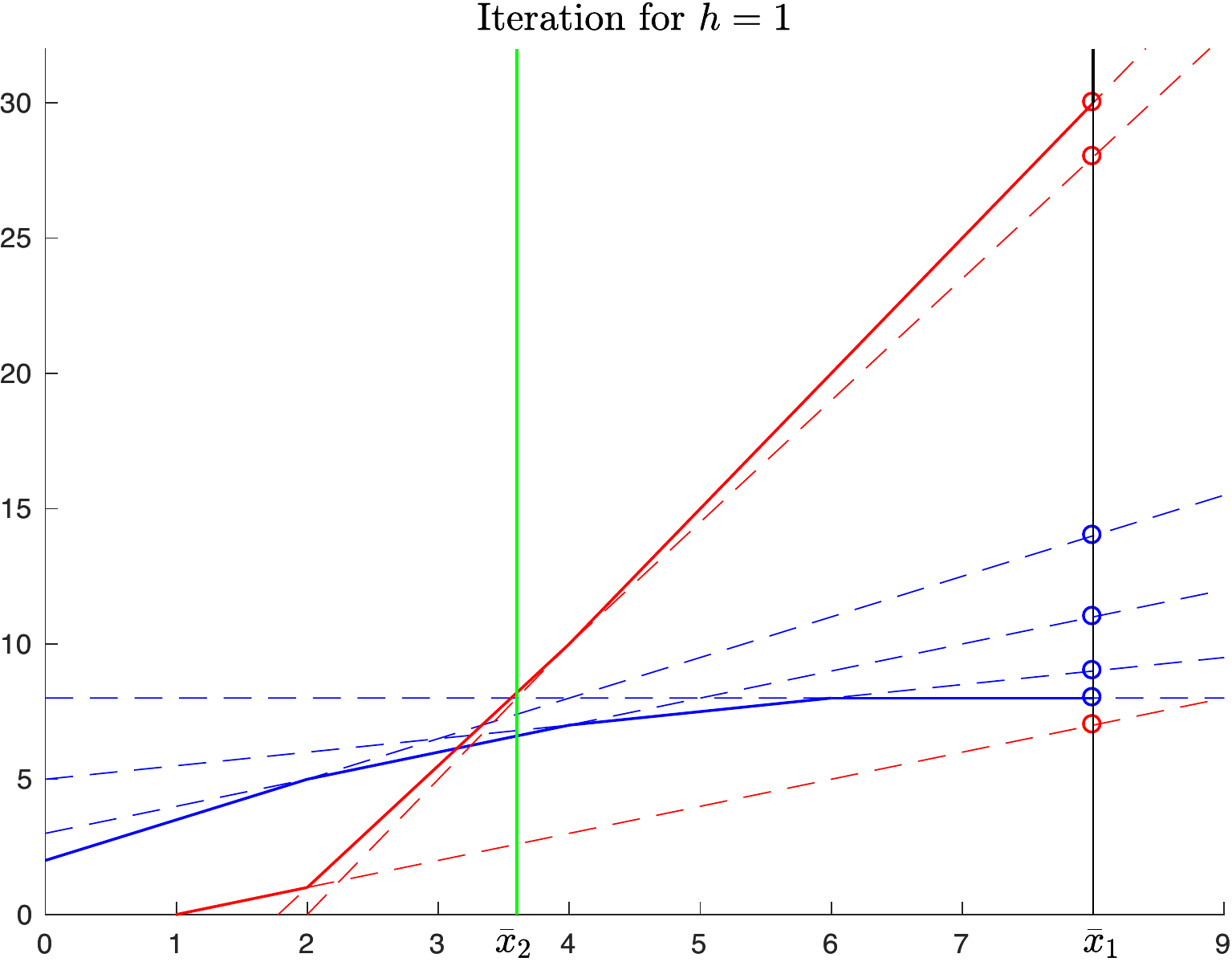}
\caption{\label{fig:h1}The first step of Algorithm \ref{Alg:convex}. The red lines represent the linear function $[f_{\bar{j}}^{i}]^{-1}$ 
while the blue ones are the functions $b_{\bar{k}^{i}}$. The green line represents 
the solution of the one dimensional equation  $f_{\bar{j}}^i(b_{\bar{k}}^i(x))=x$.}
\end{figure}
In the second iteration we have
\begin{align*}
\bar{y} &=B_i^{\bar{{\bf u}}}(\bar{x}_2)=\min_{k=1,\ldots,4}\{b_k^i(\bar{x}_2)\}=\min\left\{\frac{37}{5},\frac{33}{5},\frac{34}{5},8\right\}=\frac{33}{5},
\end{align*}
with $\bar{k}=2$, and
\begin{align*}
\bar{z} &=[F_i]^{-1}(\bar{x}_2)=\max_{j=1,\ldots,3}\{[f_j^i]^{-1}(\bar{x}_2)\}  =\max\left\{\frac{13}{5},\frac{41}{5},8\right\}=\frac{41}{5},
\end{align*}
with $\bar{j}=2$. Then, $\bar{x}_3$ is the solution of the equation $x+3=\frac{9}{2}x-8$, i.e., $\bar{x}_3=\frac{22}{7}$ (See also Figure \ref{fig:h2}).
\begin{figure}[!h]
\centering
\includegraphics[width=0.6\columnwidth]{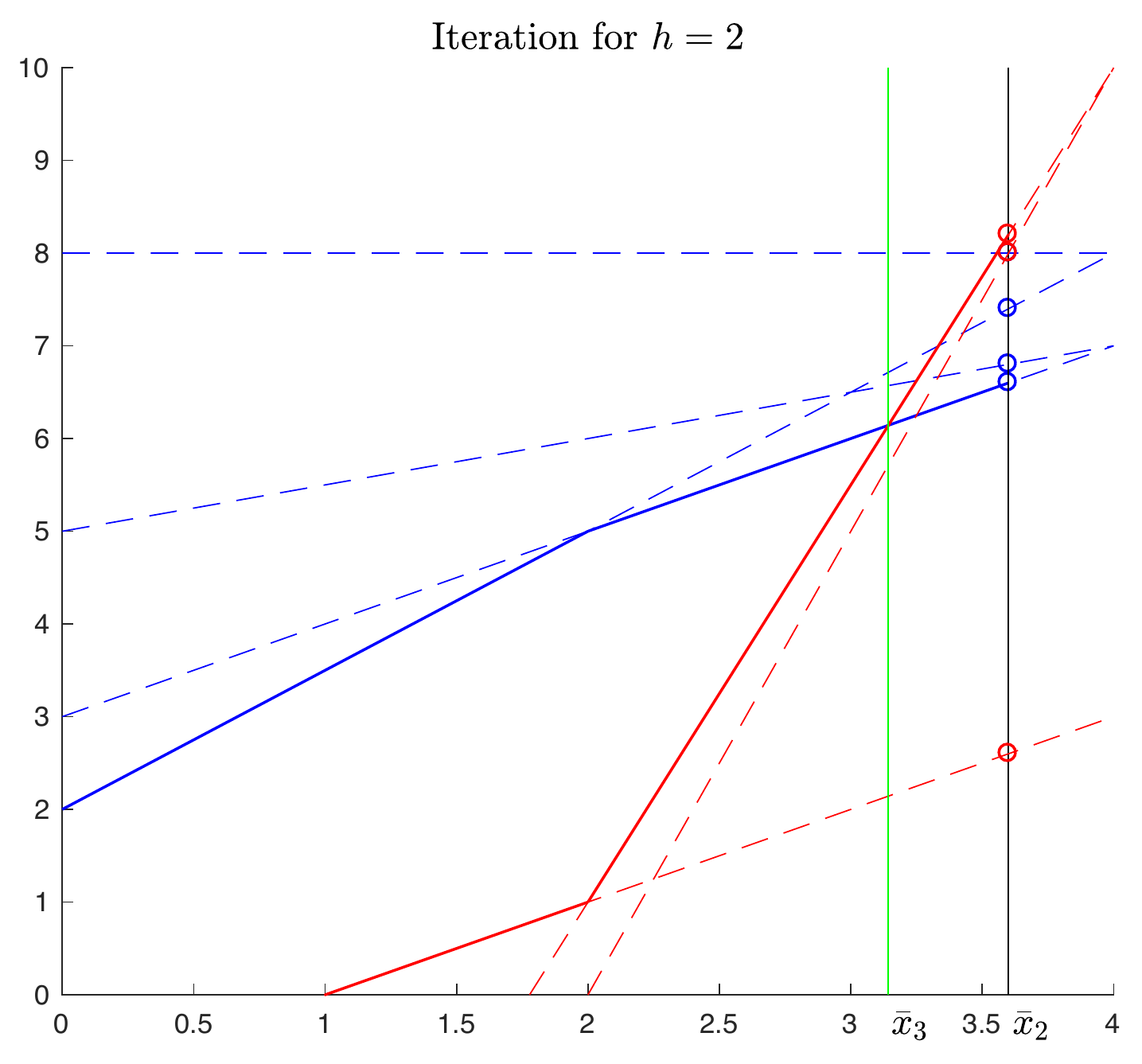}
\caption{\label{fig:h2} The second step of Algorithm~\ref{Alg:convex}.}
\end{figure}

In the third iteration we have
\begin{align*}
\bar{y} &=B_i^{\bar{{\bf u}}}(\bar{x}_3)=\min_{k=1,\ldots,4}\{b_k^i(\bar{x}_3)\} =\min\left\{\frac{47}{7},\frac{43}{7},\frac{46}{7},8\right\}=\frac{43}{7},
\end{align*}
with $\bar{k}=2$, and
\begin{align*}
\bar{z} & =[F_i]^{-1}(\bar{x}_3)=\max_{j=1,\ldots,3}\{[f_j^i]^{-1}(\bar{x}_3)\}=\max\left\{\frac{15}{7},\frac{43}{7},\frac{40}{7}\right\}=\frac{43}{7},
\end{align*}
with $\bar{j}=2$. Then, $\bar{x}_4=\bar{x}_3$ and since $\bar{y}=\bar{z}$ the algorithm stops and returns the optimal solution $\left(\frac{22}{7},\frac{43}{7}\right)$ (See also Figure \ref{fig:h3}).
\begin{figure}[!h]
\centering
\includegraphics[width=0.6\columnwidth]{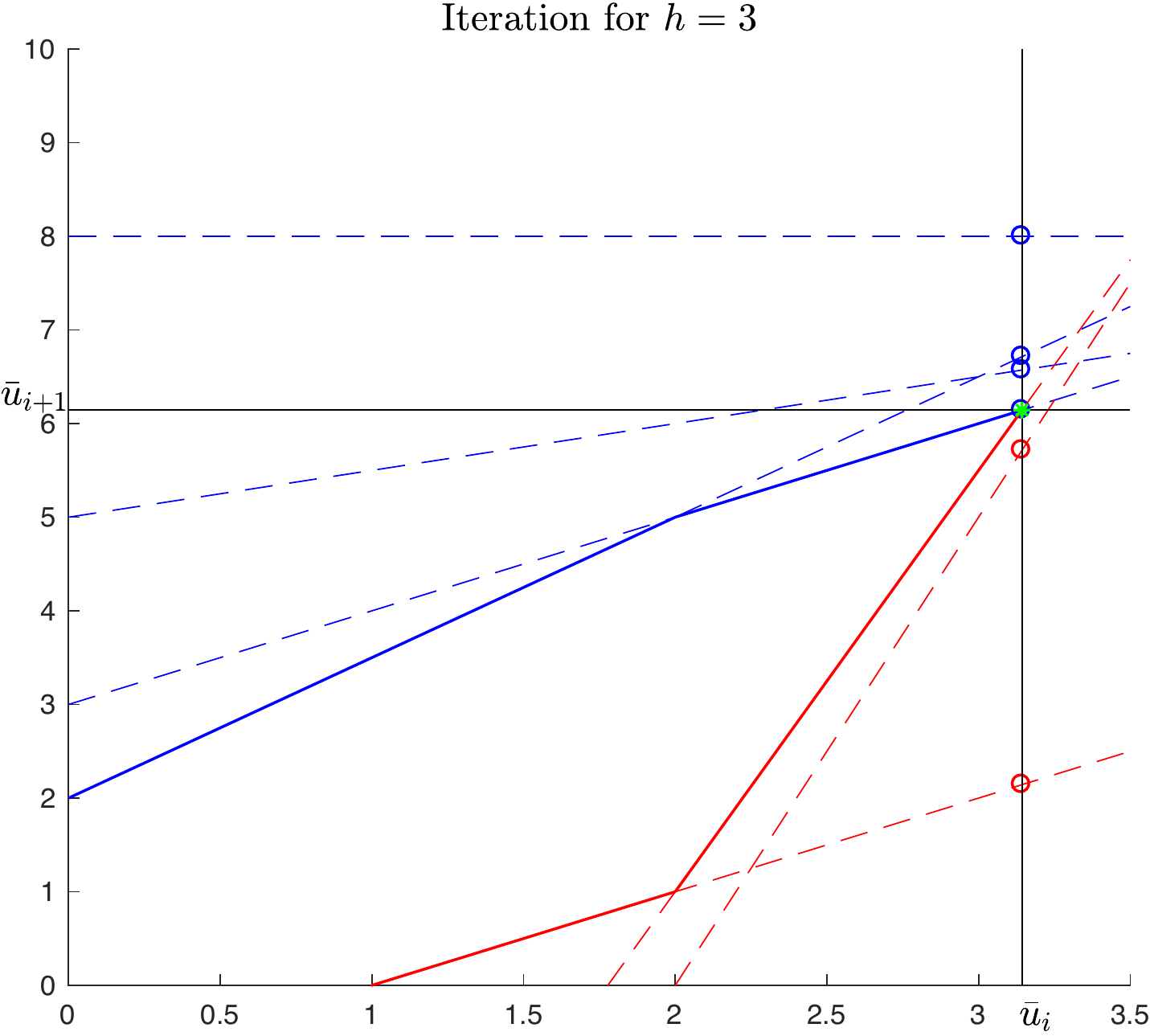}
\caption{The last step of Algorithm~\ref{Alg:convex}. The green marker represents value $(u_{i},u_{i+1}) =(\bar{x}_{h},\bar{y})$ \label{fig:h3} returned.}
\end{figure}
\end{example}

If we denote by $C_b$ the time needed to evaluate one function $b_k^i$, by $C_f$ the time needed to evaluate one function
$[f_j^i]^{-1}$, and by $C_{eq}$ the time needed to solve a one-dimensional equation $f_{j}^i(b_{k}^i(x))=x$, we can state the complexity of Algorithm \ref{Alg:convex}.
Before we need to prove one lemma.
\begin{lem}
\label{lem:1}
The sequence  $\{\bar{x}_h\}$ is strictly decreasing.
\end{lem}
\begin{proof}
If the algorithm does not stop, then $\bar{y}<\bar{z}$, or, equivalently
$$
b_{\bar{k}}^i(\bar{x}_h)- [f_{\bar{j}}^i]^{-1}(\bar{x}_h)<0.
$$
Since $[f_{\bar{j}}^i]^{-1}$ is convex,
$b_{\bar{k}}^i(x)- [f_{\bar{j}}^i]^{-1}(x)$ is concave. Moreover, $b_{\bar{k}}^i(0)- [f_{\bar{j}}^i]^{-1}(0)>0$. Then,
there exists a unique value $x\in (0,\bar{x}_h)$ such that
$$
b_{\bar{k}}^i(x)- [f_{\bar{j}}^i]^{-1}(x)=0.
$$
Such value, lower than $\bar{x}_h$, is also the solution $\bar{x}_{h+1}$ of the one-dimensional equation $f_{\bar{j}}^i(b_{\bar{k}}^i(x))=x$.
\end{proof}

Now we are ready to prove the complexity result.
\begin{prop}
Algorithm \ref{Alg:convex} has complexity $O(r_i t_i(t_i C_b+ r_i C_f + C_{eq}))$.
\end{prop}
\begin{proof}
In view of Lemma \ref{lem:1} the sequence $\{\bar{x}_h\}$ is decreasing, which means that an equation
$f_{j}^i(b_{k}^i(x))=x$, $j=1,\ldots,r_i$, $k=1,\ldots,t_i+1$, is solved at most once. Thus, the number of iterations is at most $(t_i+1)r_i$.
At each iteration we need to evaluate $B_i^{\bar{{\bf u}}}$ ($(t_i+1)C_b$ operations), evaluate $[F_i]^{-1}$ ($r_i C_f$ operations), and solve
a one-dimensional equation ($C_{eq}$ operations).
\end{proof}

Algorithm \ref{Alg:convex} and the related complexity result can be improved when the functions $b_k^i$ and $f_j^i$ are linear ones. The linear case
is particularly relevant in our context since a suitable discretization of Problem \ref{prob:2} will turn out to fall into this case, as we will see in Section \ref{sec:Disc-Problem}.
Algorithm \ref{Alg:linear} is a variant of Algorithm \ref{Alg:convex} for the linear case.
In the initialization phase of Algorithm \ref{Alg:linear} the slopes $m_k$, $k=1,\ldots, t_i+1$, of the linear functions $b_k^i$ are ordered in a decreasing way, i.e.,
$$
m_k>m_{k+1},\ \ \ k=1,\ldots,t_i,
$$
while
the slopes $\eta_j$, $j=1,\ldots, r_i$, of the linear functions $[f_j^i]^{-1}$ are ordered in a decreasing way, i.e.,
$$
\eta_j<\eta_{j+1},\ \ \ j=1,\ldots,r_i-1.
$$

Note that, in case of two linear functions with the same slope, one of the two can be eliminated since it gives rise to a redundant constraint.
The pointer $\xi$ is updated in such a way that at each iteration it identifies the index $\bar{k}$ such that
$\bar{y}=B_i^{\bar{{\bf u}}}(\bar{x}_h)=b_{\bar{k}}^i(\bar{x}_h)$, without the need of computing the value of {\em all} the functions
$b_k^i$ (as, instead, required in Algorithm \ref{Alg:convex}) and, thus, saving the $O(t_i)$ time required by this computation. Similarly,
the pointer $\phi$ is updated in such a way that at each iteration it identifies the index $\bar{j}$ such that
$\bar{z}=[F_i]^{-1}(\bar{x}_h)=[f_{\bar{j}}^i]^{-1}(\bar{x}_h)$.
\begin{algorithm}
\caption{\label{Alg:linear} Algorithm for the computation of the new values $\bar{u}_i$ and $\bar{u}_{i+1}$ in the linear case.}
%\AlgData{${\bf u}\in \mathbb{R}_+^n$, $s,t\in \{1,\ldots,n\}$ $s\leq t$;}
%\AlgResult{${\bf v}^*$}
%\tcc{                   {\em Computation of the vertices in the upper part}}
Order the slopes of the linear functions $b_k^i$, $k=1,\ldots,t_i+1$, in a decreasing way\;
Order the slopes of the linear functions $[f_j^i]^{-1}$,
$j=1,\ldots,r_i$, in an increasing way\;
Remove redundant constraints and update $t_i$ and $r_i$ accordingly\;
Set $\xi=t_i+1$ and $\phi=r_i$\;
Let $\bar{x}_1=\bar{u}_i$\;
Set $\bar{y}=-\infty$, $\bar{z}=+\infty$, $h=1$\;
\While{$\bar{y}< \bar{z}$}{
\While{$\xi>1$ and $b_{\xi-1}^i(\bar{x}_h)<b_{\xi}^i(\bar{x}_h)$}
{Set $\xi=\xi-1$\;}
Set $\bar{k}=\xi$ and $\bar{y}=b_{\bar{k}}^i(\bar{x}_h)$\;
\While{$\phi>1$ and $[f_{\phi-1}^i]^{-1}(\bar{x}_h)>[f_{\phi}^i]^{-1}(\bar{x}_h)$}
{Set $\phi=\phi-1$\;}
Set $\bar{j}=\phi$ and $\bar{z}=[f_{\bar{j}}^i]^{-1}(\bar{x}_h)$\;
Let $\bar{x}_{h+1}$ be the solution of the one-dimensional equation $f_{\bar{j}}^i(b_{\bar{k}}^i(x))=x$\;
Set $h=h+1$\;
}
\Return{$(\bar{u}_i,\bar{u}_{i+1})=(\bar{x}_h,\bar{y})$}\;
\end{algorithm}
We illustrate the algorithm on the previous example.
\begin{example}
The slopes of the functions $b_k^i$ are already ordered in a decreasing way, while those of the functions $[f_j^i]^{-1}$ are already ordered in an increasing way.
\newline
In the first iteration we immediately exit the first inner {\tt While} cycle since $b_4^i(\bar{x}_1)<b_3^i(\bar{x}_1)$, so that at the end of the cycle
we set $\bar{k}=\xi=4$. We also immediately exit the second  inner {\tt While} cycle since $[f_3^i]^{-1}(\bar{x}_1)>[f_2^i]^{-1}(\bar{x}_1)$, so that at the end of the cycle
we set $\bar{j}=\phi=3$.
\newline
In the second iteration the first inner {\tt While} cycle is repeated twice since
$$
b_4^i(\bar{x}_2)>b_3^i(\bar{x}_2)>b_2^i(\bar{x}_2)<b_1^i(\bar{x}_2),
$$ so that at the end of the cycle
we set $\bar{k}=\xi=2$. The second  inner {\tt While} cycle is repeated once since
$$
[f_3^i]^{-1}(\bar{x}_1)<[f_2^i]^{-1}(\bar{x}_2)>[f_1^i]^{-1}(\bar{x}_2),
$$
so that at the end of the cycle
we set $\bar{j}=\phi=2$.\newline
In the third iteration we immediately exit the first inner {\tt While} cycle since $b_2^i(\bar{x}_1)<b_1^i(\bar{x}_3)$, so that at the end of the cycle
we set $\bar{k}=\xi=2$. We also immediately exit the second  inner {\tt While} cycle since $[f_2^i]^{-1}(\bar{x}_3)>[f_1^i]^{-1}(\bar{x}_3)$, so that at the end of the cycle
we set $\bar{j}=\phi=2$.
\end{example}
The following proposition establishes the complexity of Algorithm \ref{Alg:linear}.
\begin{prop}
Algorithm \ref{Alg:linear} has complexity $O(r_i \log(r_i) + t_i \log(t_i))$.
\end{prop}
\begin{proof}
We first remark that the initial orderings of the slopes already require
the computing time $O(r_i \log(r_i) + t_i \log(t_i))$, while removing
redundant constraints requires $O(t_i + r_i)$ time. Next, we remark that
in the linear case $C_b, C_f$ and $C_{eq}$ are $O(1)$ operations. In particular, the one-dimensional equation is a linear one. Moreover, we notice that
$B_i^{\bar{{\bf u}}}$ is a concave piecewise linear function, while $[F_i]^{-1}$ is a convex piecewise linear function.
Since in view of Lemma \ref{lem:1} the sequence  $\{\bar{x}_h\}$ is decreasing, the corresponding sequence of slopes of the function $B_i^{\bar{{\bf u}}}$ at points $\bar{x}_h$ is not decreasing, while
the sequence of slopes of the function $[F_i]^{-1}$ is not increasing, and at each iteration at least one slope must change (otherwise we would solve the same linear equation
and $\bar{x}_h$ would not change). Then, the number of different slope values and, thus, the number of iterations, can not be larger than $t_i+r_i+1$.
Moreover, by updating the two pointers $\xi$ and $\phi$ , the overall number of evaluations of the functions $b_k^i$ and $[f_j^i]^{-1}$, needed to compute
the different values $\bar{y}$ and $\bar{z}$ in the outer {\tt While} cycle can not be larger than $O(t_i+r_i)$. Consequently, the computing time of
the outer {\tt While} cycle is $O(t_i+r_i)$ and the complexity of the algorithm is determined by the initial orderings of the slopes.
\end{proof}

While in practice we employed Algorithm \ref{Alg:linear} to compute $\bar{u}_i, \bar{u}_{i+1}$, in the linear case
we could also solve the linear subproblem (\ref{eq:convex sub}). This can be done in linear time $O(t_i+r_i)$ with respect to the
number of constraints, e.g., by Megiddo's algorithm (see \citep{Megiddo83}). Thus, we can state the following complexity result for Problem
(\ref{eq:probl}) in the linear case.
\begin{thm}
If $f_j^i$, $b_k^i$, $i=1,\ldots,n$, $j=1,\ldots,r_i$, and $k=1,\ldots,t_i$, are linear functions, then Problem
(\ref{eq:probl}) can be solved in time $O(\sum_{i=1}^n (t_i+r_i))$ by Algorithm~\ref{Alg:FindProfile3}, if $\bar{u}_i, \bar{u}_{i+1}$
are computed by Megiddo's algorithm. Such complexity is optimal.
\end{thm}
\begin{proof}
The complexity result immediately follows by the observation that the most time consuming part of Algorithm~\ref{Alg:FindProfile3}
is the forward one with the computation of $\bar{u}_i, \bar{u}_{i+1}$. Indeed, the backward part is run in $O(\sum_{i=1}^n t_i)$ time (at each iteration
we only need to evaluate $B_i$).
The fact that such complexity is optimal follows from the observation that $O(\sum_{i=1}^n (t_i+r_i))$ is also the size of the input values for the problem.
\end{proof}
%\subsection{Removing the concavity assumption}
%The approach described above can be applied also when the concavity assumption for $f_j^i$ and $b_k^i$ is not fulfilled.
%First, we remark that Proposition \ref{prop:2} still holds even without the concavity assumption.
%The only thing we need to change in the previous development is the
%definition of the new value $\bar{x}_{h+1}$ in Algorithm \ref{Alg:convex}. We need to define it as follows:
%\begin{equation}
%\label{eq:fixedpoint}
%\bar{x}_{h+1}=\sup\{x\in [0,\bar{x}_h]\ :\ f_{\bar{j}}^i(b_{\bar{k}}^i(x))=x\},
%\end{equation}
%since after removing the concavity assumption we can not guarantee the existence of a unique positive fixed point for
%$f_{\bar{j}}^i\circ b_{\bar{k}}^i$.
%We also have some changes in the complexity analysis. The term
%$r_i t_i C_{eq}$ should be replaced with a term $C_{fixed} \sum_{j=1}^{t_i}\sum_{k=1}^{r_i} R_{jk}^i$, where
%$C_{fixed}$ is the time needed to solve a single problem (\ref{eq:fixedpoint}), while $R_{jk}^i$ denotes the number of fixed points (in the interval $[0,u_i]$)
%for $f_{j}^i\circ b_{k}^i$. Note that when the concavity assumption holds, $R_{jk}^i\leq 1$ and $C_{fixed}=C_{eq}$, so that the previous complexity result is recovered.

\section{Discretization of the speed-planning problem}\label{sec:Disc-Problem}
Problem~\ref{prob:disc} does not belong to the class defined
in~(\ref{eq:probl}). In this section, we show that a
small variation of Problem~\ref{prob:2}, followed by discretization, allows obtaining a problem that belongs to class~(\ref{eq:probl}).

To this end, consider the following family of problems, depending on the
positive real parameter~$h$.
\begin{problem}\label{prob:offset}
\begin{align}
\displaystyle\min_{a,\btau \in C^{0}, b \in C^{1}}&  \displaystyle\int_0^{s_f} b(s)^{-1/2} \, ds, \label{con:obj-5}\\
\mbox{\small\textrm{subject to}} &  \ (\forall s \in [0,s_{f}], \ j=1,\dots,p) \nonumber\\
&b(s + \lambda_{j}(s)) = \hat{b}_{j}(s) \\
&b(s + \eta_{j}(s)) = \bar{b}_{j}(s) \\
& d_{j}(s)a(s) + c_{j}(s) \hat{b}_{j}(s) + g(s) = \tau_{j}(s), \label{con:dynamic-5} \\
&  \gamma^{\prime}_{j}(s)a(s)+ \gamma^{\prime\prime}_{j}(s)\bar{b}_{j}(s) = \ddot{q}_{j}(s), \label{con:gamma_5}\\
& b^\prime(s) = 2a(s), \label{con:der-5}\\
& \lvert \ddot{\bq}(s) \rvert \le \balpha (s), \label{con:acc-5}\\
&  0 \le  \bgamma^{\prime}(s)^{2} b(s) \le \bpsi(s)^{2}, \label{con:vel-5}\\
& \lvert  \tau(s) \rvert  \le \bmu(s), \label{con:force-5}\\
& b(0) = 0, \, b(s_f) =0,  \label{con:inter-5}\\
%&a,\tau \in C^{0}, b \in C^{1}.
\end{align}
\end{problem}
where functions $\lambda_{j}$ and $\eta_{j}$ are defined as follows ($\forall s \in[0,s_{f}]$ and $j = 1,\dots,p$):
\begin{equation}\label{eq:lambda-cont}
\lambda_{j}(s) = \begin{cases}
h, & d_{j}(s)c_{j}(s) \ge 0\\
0, & d_{j}(s)c_{j}(s)  \le 0,
\end{cases}
\end{equation}
\begin{equation}\label{eq:eta-cont}
\eta_j(s) = \begin{cases}
h, & \gamma^{\prime}_{j}(s)\gamma^{\prime\prime}_{j}(s) \ge 0\\
0, &\gamma^{\prime}_{j}(s)\gamma^{\prime\prime}_{j}(s)  \le 0\,.
\end{cases}
\end{equation}

Note that, for $h=0$, Problem~\ref{prob:offset} becomes Problem~\ref{prob:2}.
Further, for every $h>0$, Problem~\ref{prob:offset} has an optimal
solution (this can be proved with the same arguments used for
Proposition~\ref{prop-solution-existence}). Let $b^*_h$ be the
solution of Problem~\ref{prob:offset} as a function of $h$.
Note that, by~\eqref{con:gamma_5} and~\eqref{con:acc-5},
$\forall h>0, \forall s \in [0,s_f], |b'(s)| \leq L = 2\sqrt{p}(\|\balpha\| + p \|\bgamma^{\prime\prime}\|\|\bpsi\|^{2})$
%leq L= 4 \max_{s \in
%  [0,s_f]} \{\alpha(s)+\|\gamma''(s)\|_{\infty} \psi(s)\}$,CONTROLLARE QUESTA FORMULA !!!!!!!!!!!!!!!!!!!!!!!!!!!!!!!!!!!!!
so  that $b^*_h$
is Lipshitz with constant $L$, independent of $h$. Thus, Ascoli-Arzel\`{a}
Theorem implies that from any succession of solutions $b^*_{h_i}$,
with $\lim_{i \to \infty} h_i=0$ we can extract a convergent
subsequence that converges to a solution of Problem~\ref{prob:2}.

%Problem~\ref{prob:4} is a discretization of Problem~\ref{prob:offset}.
%Increasing the number of samples $n$, the solutions of Problem~\ref{prob:4} are better approximations of the
%solutions of Problem~\ref{prob:offset}.  Moreover, $b$ is Lipschitz, since  $b \in C^{1}([0,s_{f}],\Real)$ and the interval $[0, s_{f}]$ is a compact set, hence, as $h$ approaches to $0$, the solution of Problem~\ref{prob:offset} is also an approximated solution of Problem~\ref{prob:2}.

%if constraints  (\ref{con:dynamic-disc})-(\ref{con:acc-disc}) are evaluated with a  different discretization
%approach.

Discretizing Problem~\ref{prob:offset} with step $h$, we obtain the following problem.

\begin{problem}\label{prob:4}
\begin{align}
\displaystyle\min_{\ba,\bbold,\btau}& \ 2h\displaystyle\sum_{i=1}^{n-1} \left( \frac{1}{b_i^{1/2} + b_{i+1}^{1/2}} \right),\label{con:obj-4}\\
\mbox{\small\textrm{subject to}} &  \ (i=1,\dots,n-1)\\
& \blambda_{i} b_{i+1} + (\be-\blambda_{i})b_{i} = \hat{\boldsymbol{b}}_{i}, \label{con:convex-1}\\
& \etab_{i} b_{i+1} + (\be-\etab_{i})b_{i} = \bar{\boldsymbol{b}}_{i}, \label{con:convex-2}\\
& \bd_i a_{i}\be + \bc_i  \hat{\boldsymbol{b}}_{i} + \bg_i = \btau_i, \label{con:dynamic-4} \\
& \bgamma^{\prime}_i a_{i}\be+ \bgamma^{\prime\prime}_i\bar{\boldsymbol{b}}_{i} = \ddot{\bq}_{i}\\
&b_{i+1} - b_{i} = 2 a_i h ,&  \label{con:approx-4} \\
& \lvert \btau_i  \rvert \le \bmu_i, \label{con:force-4}\\
&  0 \le  \bgamma^{\prime 2}_{i}b_i \le \bpsi_{i}^{2}, &  \label{con:vel-4}\\
& \lvert \ddot{\bq}_{i} \rvert \le \balpha_i, \label{con:acc-4}\\
& b_1 = 0, \, b_n =  0. \label{con:inter-4} \\
&\boldsymbol{b} \in \Real^{n} \ \ba \in\Real^{n-1},\btau_{i}\in\Real^{p}\,.
\end{align}
\end{problem}
%Note that equations(\ref{con:convex-1})-(\ref{con:convex-2}) correspond to a selection technique
%between the variables $b_{i}$ and $b_{i+1}$.
Here, for $j = 1,\dots,p$,
$i=1,\dots,n$, \[
\lambda_{j,i} =
\begin{cases}
1, &  d_{j,i}  c_{j,i} \ge 0 \\
0, &  d_{j,i} c_{j,i}  < 0\,,
\end{cases}
\]
\[
\eta_{j,i} =
\begin{cases}
1, &  \bgamma_{j,i}^{\prime}  \bgamma^{\prime\prime}_{j,i} \ge 0 \\
0, &  \bgamma_{j,i}^{\prime}  \bgamma^{\prime\prime}_{j,i}  < 0\,.
\end{cases}
\]

The following proposition will be proved in the appendix.
\begin{prop} \label{prop:selection}
Problem~\ref{prob:4} belongs to problem class (\ref{eq:probl}).
\end{prop}

%Since the objective function~(\ref{con:obj-4}) is strictly monotone decreasing with respect to $b_{i}$, reference \citep{Nagy2017} shows that, due to Proposition~\ref{prop:selection}, it is possible to substitute the objective
%function~(\ref{con:obj-4}) with $\max \sum_{i=0}^{n} b_{i}$ (i.e., maximize te velocity along the path) obtaining a Linear Programming Problem (LP).\\
%

\subsection{Comparison with TOPP-RA algorithm (\citep{pham2017new})}
\label{sec_comparison}

As already mentioned in the introduction, a very recent and
interesting work, closely related to ours, is~\citep{pham2017new}. In
that paper a backward-forward approach is proposed. In the backward
phase a controllable set is computed for each discretization
point. This is an interval that contains all possible states for which
there exists at least a sequence of controls leading to the final
assigned state. The computation of each interval requires the solution
of two LP problems with two variables. Next, a forward phase is
performed where a single LP with two variables
is solved for each discretization point. The final result is a
feasible solution which, however, is optimal under the assumption that
no zero-inertia points are present. In the presence of zero-inertia
points a solution is returned whose objective function value differs
from the optimal one by a quantity proportional to the discretization
step $h$. The overall number of two-dimensional LPs solved by this
approach is $3n$, while in our approach we solve in total only $n$
LPs. In~\citep{pham2017new} the LPs are solved by the simplex method
while we proposed an alternative method which turns out to be more
efficient. Indeed, our computational experiments will show that the
computation times are reduced by at least an order of magnitude when using our alternative method. In~\citep{pham2017new} it is observed that the practical (say, average) complexity of the simplex method is linear with respect
to the number of constraints. In fact, we observed that for two-dimensional LPs such complexity is not only the practical one but also the worst-case one.
Finally, in our approach we deal with the presence of zero-inertia
points through the addition of the displacements
(\ref{eq:lambda-cont})-(\ref{eq:eta-cont}). Introducing these
displacement, we are able to return an exact solution of the
discretized problem.

\section{Experimental results}
\label{sec:demo}
In this section, we consider a motion planning problem for a 3-DoF
manipulator and compare the computation time of the proposed
solver to other methods existing in literature. We also show an experiment on the execution
of a time-optimal velocity profile on a 6-DoF robotic manipulator.

\subsection{Test case on a 3-DOF manipulator}
We consider the robot  presented in~\citep{murray1994mathematical} (Chapter 4, example 4.3).
 This robot is a serial chain robot (see Figure~\ref{fig:3dof}), composed of 3 links connected 
 with 3 revolute joints (the first link is connected with a fixed origin).
Table~\ref{tab:demo_dyn} reports the robot parameters.
\begin{figure}
	\centering
		\includegraphics[height=0.35\columnwidth]{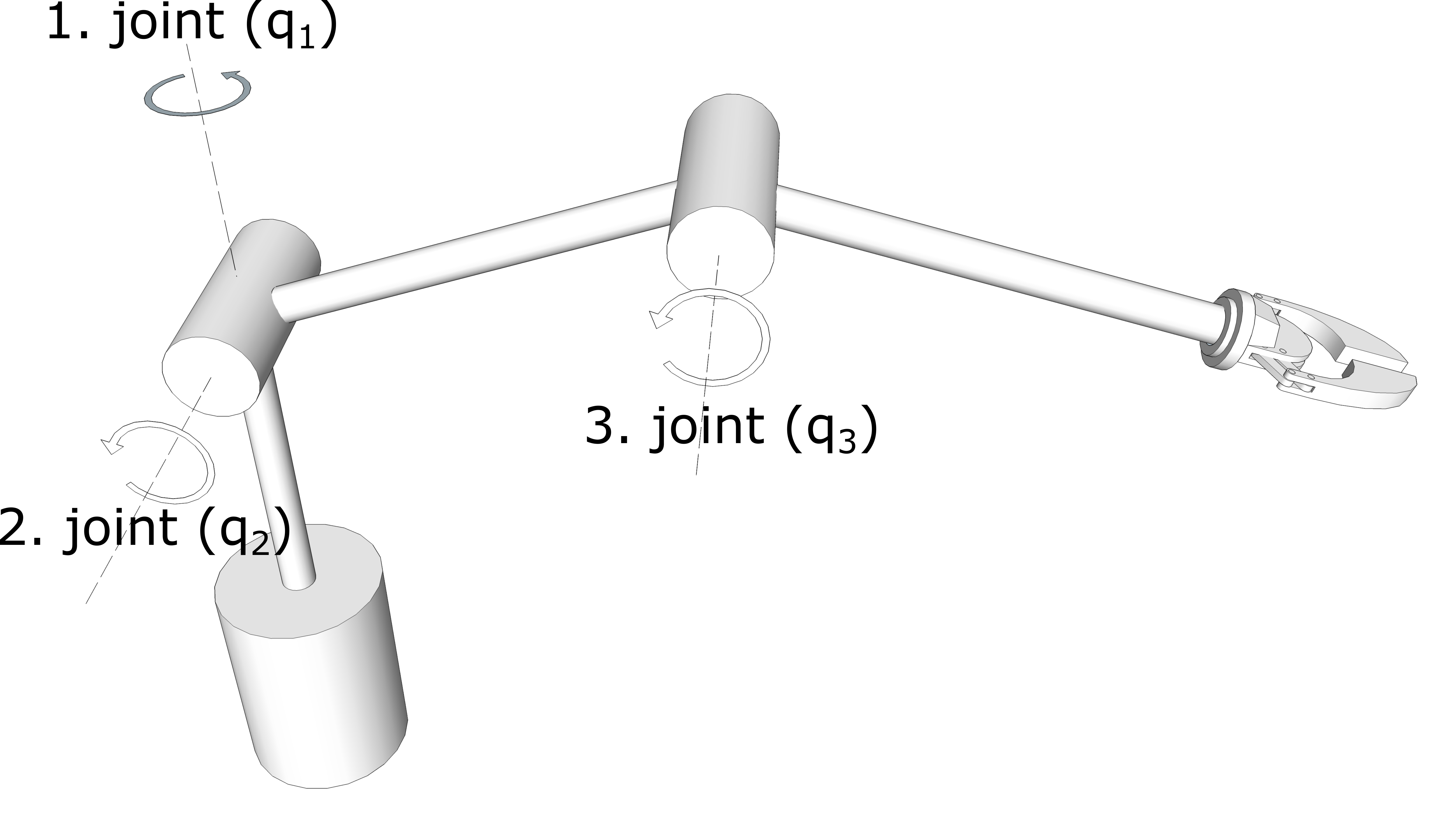}
		\caption{The 3-DoF manipulator with the three revolute joints ($\bq = [q_1, \  q_2, \ q_3]^{T}$). } \label{fig:3dof}
	
\end{figure}
Namely, for link
$i$, $i=1,\ldots,3$, $l_i$
is the length, and $r_i$ is the distance between the gravity center of
the link and the joint that connects it to the previous link in the
chain (see Figure \ref{fig:3dof_2d}). Parameters $I_{xi},I_{yi},I_{zi}$, $m_i$ are the diagonal components of
the inertia matrix and the mass of link $i$.
\begin{table}
	\small\sf\centering
		\caption{Kinematic and dynamic parameters for the 3-DoF manipulator.} 	\label{tab:demo_dyn}
	
		\tabcolsep=0.11cm
		\begin{tabular}{ |c|c|c|c|c| }
			\hline
			Link & ($I_{xi}$, $I_{yi}$, $I_{zi}$) [$kg$ $m^2$] & $m_i$ [$kg$] & $l_i$ [$m$] & $r_i$ [$m$] \\
			\hline	
			1 & (7.5, 7.5, 7.5) & 1.5 & 0.2 & 0.08 \\
			2 & (5.7, 5.7, 5.7) & 1.2 & 0.3 & 0.12 \\
			3 & (4.75, 4.75, 4.75) & 1.0 & 0.325 & 0.13 \\						
			\hline
		\end{tabular}
	
\end{table}

\begin{figure}
	\centering
		\includegraphics[height=0.4\columnwidth]{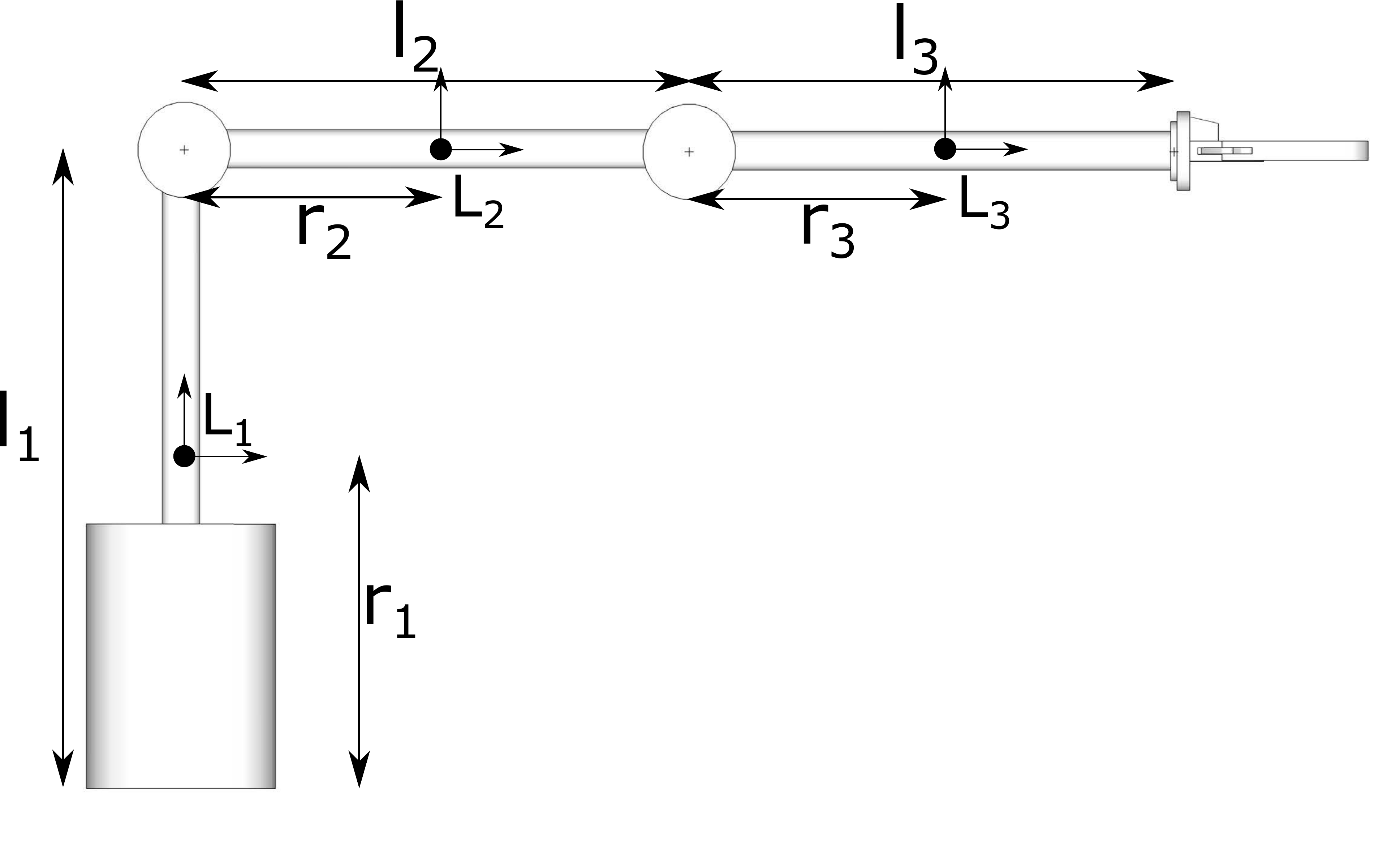}
		\caption{Kinematic and dynamic parameters for the 3-DoF manipulator. $L_i$ indicates the coordinate frame attached to the gravity center of the link. } \label{fig:3dof_2d}
	
\end{figure}

We consider an instance of Problem~\ref{prob:2}, where
the reference curve $\bgamma: [0,1] \to \Real^{3}$ is defined as a cubic spline that interpolates the points shown in
Table~\ref{tab:interpolated-point1}.
The mass matrix $\bD$,  the Coriolis matrix $\bC$ and the external
forces term $\bg$ that we consider are reported
in~\citep{murray1994mathematical} (Chapter 4, example 4.3).

\begin{table}[!h]
\caption{Points interpolated in the configuration space\label{tab:interpolated-point1}.}
\centering
\begin{tabular}{|c|c|c|c|}
\hline
$s$ & $\gamma_{1}$   & $\gamma_{2}$ & $\gamma_{3}$ \\
\hline
0 & 0 & 0 & 0\\
0.25 & 1.288 & -0.2864 & -0.2982\\
0.5 & 2.59 & -0.03045 & -0.5995\\
0.75 & 4.374 & -0.04647 & -0.582\\
1 & 5.334 & -0.1657 & -0.4504 \\
\hline
\end{tabular}
\end{table}
The following kinematic and dynamic bounds are applied for the presented test case ($\forall s \in [0,1]$)
%\begin{figure}[!h]
%\centering
%\includegraphics[width=\columnwidth]{fig/simulation.eps}
%\caption{Reference Path simulation data.\label{fig:simulation-data}}
%\end{figure}

\[
\begin{array}{l}
\bpsi(s) =  [2.0, \  2.0, \ 2.0]^{T} \\ [8pt]
\balpha(s) =  [1.5, \ 1.5, \ 1.5]^{T}  \\ [8pt]
\bmu(s) = [9, \ 9, \ 9]^{T}.
\end{array}
\]
%For every method the same constraints are applied, and a 3-DoF robot model is considered with the following parameters

\subsection{Computational time comparison}
We find an approximated solution of Problem~\ref{prob:2} by solving Problem~\ref{prob:4}
with four different methods.

\begin{enumerate}
\item a SOCP solver which solves the SOCP reformulation presented in~Equation (74)-(86) of~\cite{verscheure09};
\item a LP solver which solves the LP reformulation presented in~Equation (23) of~\cite{Nagy2017};
\item Algorithm~\ref{Alg:FindProfile3} using simplex method to solve the two-dimensional LP subproblems (\ref{eq:convex sub}).
\item Algorithm~\ref{Alg:FindProfile3} using Algorithm~\ref{Alg:linear} to solve the two-dimensional LP subproblems (\ref{eq:convex sub}).
\end{enumerate}
In the first and second method we use Gurobi solver~\cite{gurobi} while for the other methods we use a C++ implementation of Algorithm~\ref{Alg:FindProfile3}. We measure the performance on a 2.4 GHz
 Intel(R) Core(TM) i7-3630QM CPU.
\begin{figure}
	\centering
		\includegraphics[height=0.4\columnwidth]{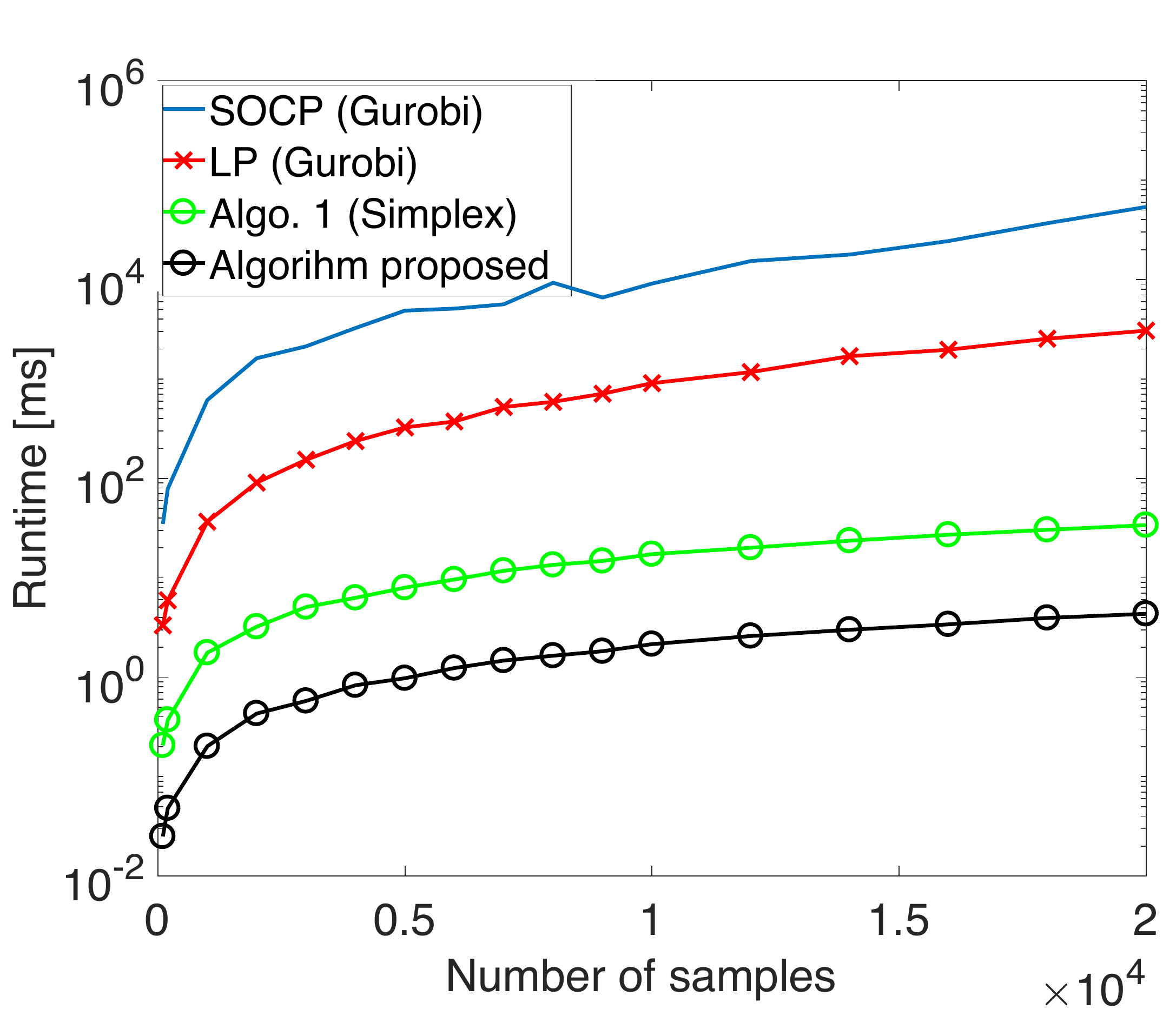}
		\caption{Computation times obtained using different
                  approaches as a function of the number of variables. } \label{fig:runtimes}
	
\end{figure}

The results presented in Figure~\ref{fig:runtimes} show that the Algorithm~\ref{Alg:FindProfile3} with
Algorithm \ref{Alg:linear} employed to solve the two-dimensional LP
subproblems significantly outperforms the other methods (in
particular, by more than four, two, and one of order of magnitude,
respectively). We did not compute directly the solution with the
TOPP-RA algorithm presented in~\cite{pham2017new}, however, note that
the computational time of TOPP-RA is comparable with
Algorithm~\ref{Alg:FindProfile3} using simplex method to solve the
two-dimensional LP subproblems (\ref{eq:convex sub}) (actually higher, since TOPP-RA solves 3$n$ LP problems, while our approach
solves only $n$ LP problems).

\subsection{Experimental results}

This section presents an experiment on a minimum-time trajectory
tracking task with a 6-DoF Mitsubishi RV-3SDB industrial robotic arm (see Figure \ref{fig:demo_robot}). We require the end-effector to track an assigned path. To this end, we compute a corresponding trajectory $\Gamma$ in the joint space and optimize the speed law on $\Gamma$ by solving Problem~\ref{prob:4} with Algorithm~\ref{Alg:FindProfile3}, using Algorithm \ref{Alg:linear} to solve the LP subproblems~(\ref{eq:convex sub}).

%while the joints positions are measured to verify the velocity profile of the robot.

\begin{figure}
	\centering
		\includegraphics[height=0.4\columnwidth]{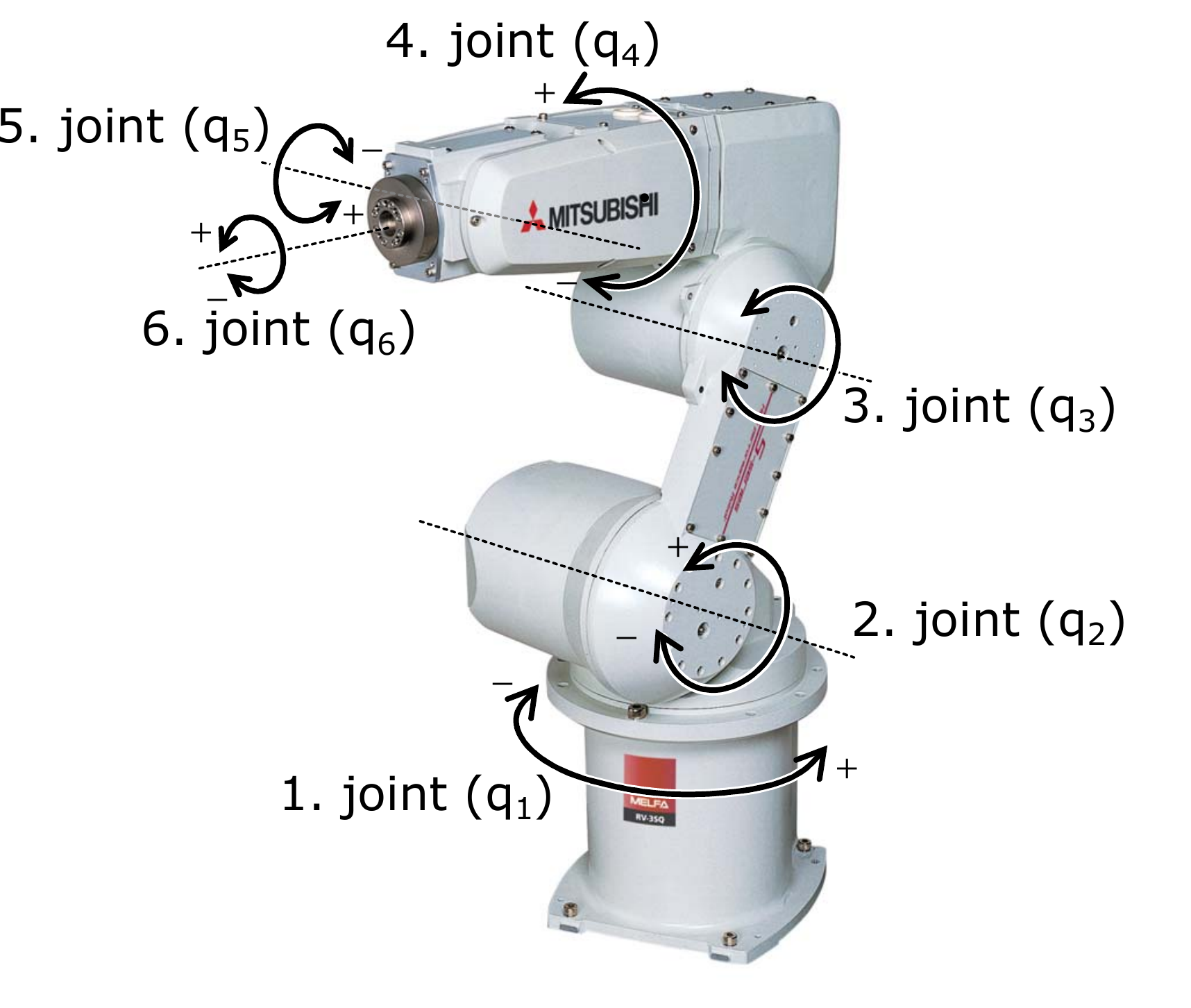}
		\caption{Mitsubishi RV-3SDB robot manipulator \cite{melfa_dat} ($\bq = [q_1, \  q_2, \ q_3, \ q_4, \ q_5, \ q_6]^{T}$).} \label{fig:demo_robot}
	
\end{figure}

The reference path of the end-effector is generated using V-REP robot simulator software~\cite{vrep13}. The path is defined by a Bezier curve with 6 control points (see Table~\ref{tb:demo_points}). In the next step, the Bezier curve is sampled in 100 points (see Figure~\ref{fig:demo_path}), which are transformed to joint space using an inverse kinematics method implemented in Robotics Toolbox for Matlab \cite{corke11}. The obtained configurations in joint space are interpolated  by a cubic spline, obtaining the reference path $\bgamma$. Derivatives $\bgamma^{\prime}$, $\bgamma^{\prime\prime}$ are calculated analytically from the spline coefficients. The reference path is sampled in $n=1000$ points.

\begin{table}[h]
	\caption{The control points of the path in V-REP.}  \label{tb:demo_points}

	\small\sf\centering
	\begin{tabular}{lll}
		\hline
		X[m] & Y[m] & Z[m]\\
		\hline
			
		0.247 & 0.443 & 0.737 \\
		0.402 & 0.335 & 0.704 \\
		0.502 & 0.360 & 0.654 \\
		0.495 & 0.195 & 0.618 \\
		0.603 & -0.178 & 0.537 \\
		0.498 & -0.398 & 0.537 \\		
	\end{tabular}
\end{table}

\begin{figure}
	\centering
		\includegraphics[height=0.4\columnwidth]{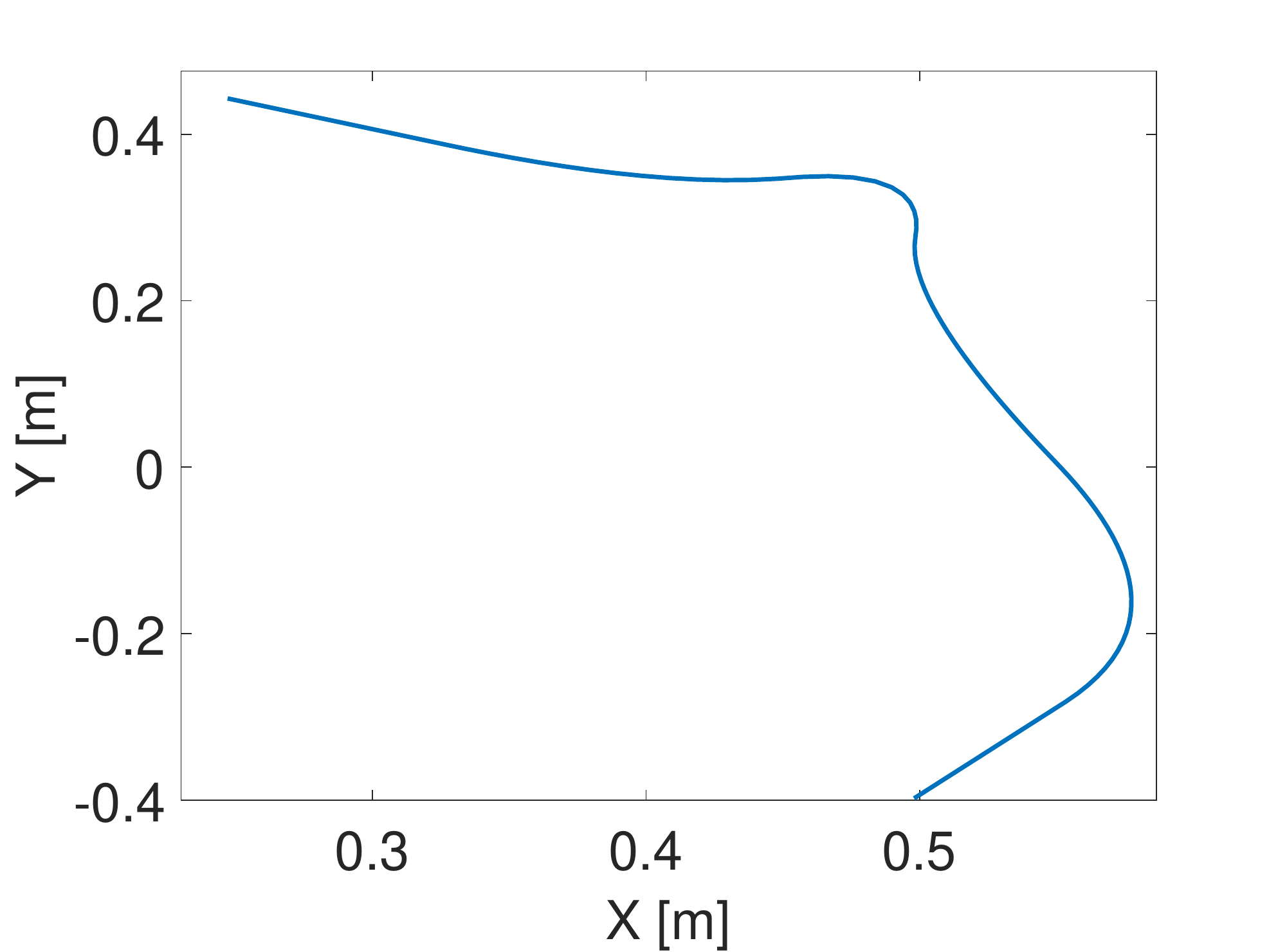}
		\caption{The reference path in the X-Y space.} \label{fig:demo_path}
	
\end{figure}

The following velocity and acceleration constraints are used for the six joints of the robot: ($\forall s \in [0,s_{f}]$)
\[
\begin{array}{l}
\bpsi(s) =  [1.0, \dots, 1.0]^{T}, \\ [8pt]
\balpha(s) = [4.0, \dots, 4.0]^{T} .
\end{array}
\]

We implemented Algorithm~\ref{Alg:FindProfile3} in C++. Velocity profile calculation takes less than 250 $\mu$s. Figure~\ref{fig:demo_profile} shows the generated velocity profile.

\begin{figure}[!h]
\centering
\includegraphics[width=0.6\columnwidth]{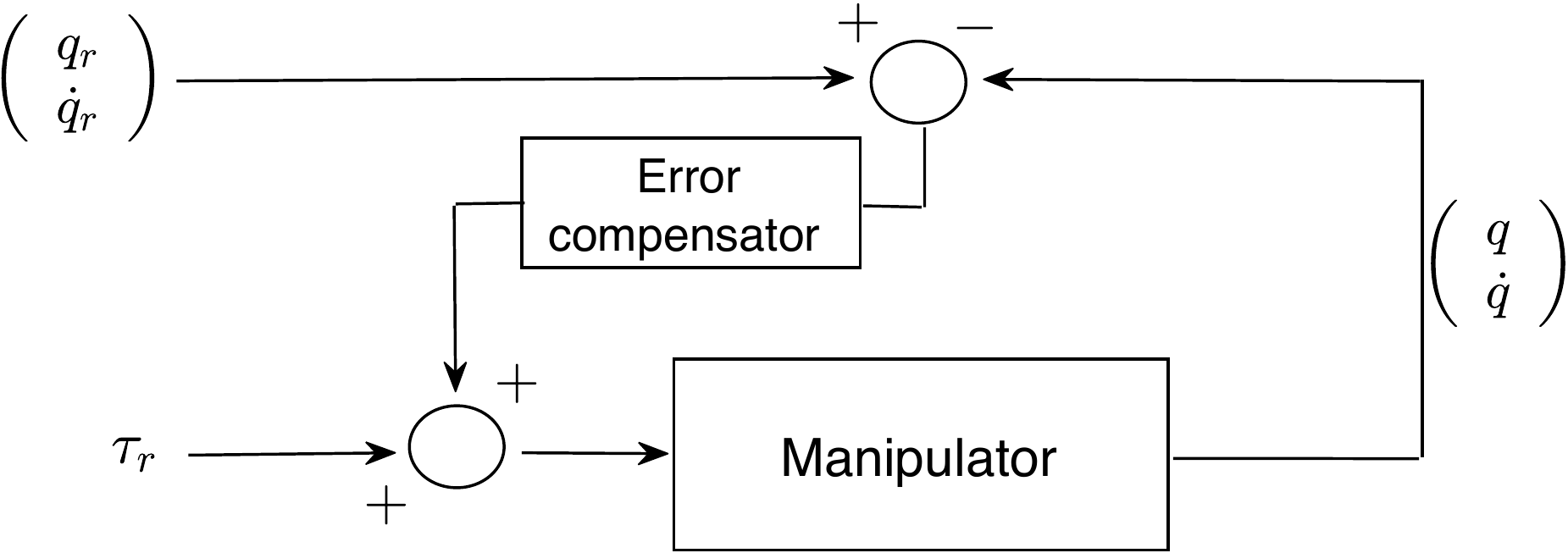}
\caption{Standard tracking control scheme.}\label{fig:trackingScheme}
\end{figure}

In order to compute a reference trajectory $(\bq_r,\dot \bq_r):
[0,s_f] \to T \mathcal{Q}$ and a reference input torque $\btau_r:[0,s_f] \to
\Real^p$ the following procedure can be applied. After the
discretized Problem~\ref{prob:4}
is solved, the continuous time
functions $\btau_c=\mathcal{I}_{\btau}$ and $b_c=\mathcal{I}_{\bbold}$ are obtained interpolating
$\btau$ and $\bbold$ as in Proposition~\ref{prop_interp_cond}.
Then, the position function $\lambda:[0,t_f] \to [0, s_f]$ is computed as the
solution of the differential equation
\[
\begin{array}{ll}
\dot \lambda(t) = \sqrt{b_c(\lambda(t))}\\
\lambda(0)=0\,.
\end{array}
\]
Finally, the reference trajectory and input are defined as
\[
\begin{array}{lll}
\bq_r(t) = \bgamma(\lambda(t)),\\[8pt]
\dot{\bq_r}(t) =  \bgamma^{\prime}(\lambda(t)) \dot
                   \lambda(t),\\[8pt]
\btau_r(t)=\btau_c(\lambda(t))\,.
\end{array}
\]

Then, one can achieve asymptotic exact tracking by the standard state
tracking control scheme shown in~Figure~\ref{fig:trackingScheme}, in which the reference torque
$\btau_r$ enters as a feedforward control signal. Due to the limitations of the available hardware, we used
a simplified control scheme.
Namely, the implemented controller is a simple position setpoint regulator,
where $\bq_r$ is used as a time-varying position reference signal  (see
Figure~\ref{fig:demo_block}).  

In particular, the robot is controlled using
Mitsubishi Real-time external control capability \cite{Melfa_inst}.
In this control scheme, the robot controller receives the time-varying
setpoint position $\bq_r$ from a PC via Ethernet
communication. The controller sends back to the PC various monitor data
(e.g., measured joint position, motor current). The controller sample rate
is $7.1$ ms \cite{Melfa_inst}.
\begin{figure*}
	\centering
		\includegraphics[width=0.7\columnwidth]{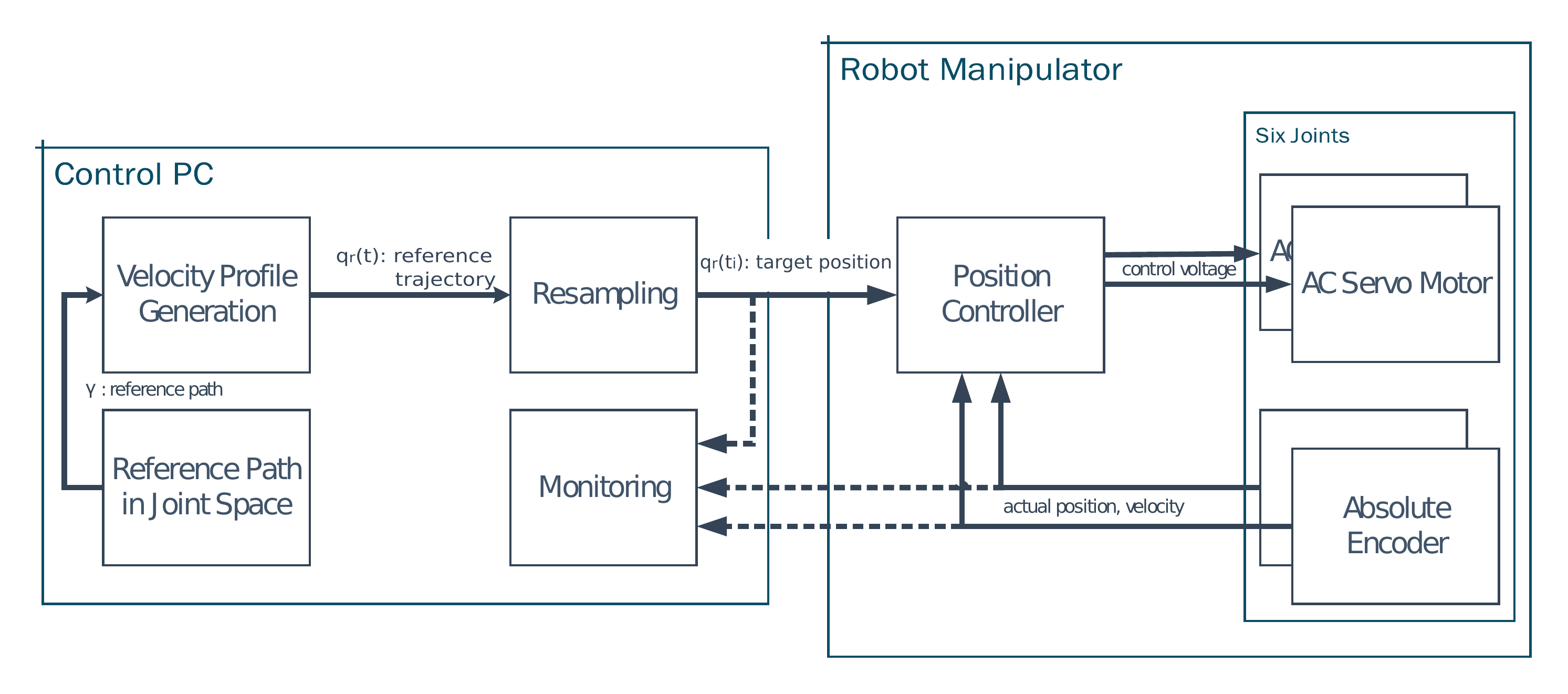}
		\caption{The implemented control scheme for trajectory tracking.} \label{fig:demo_block}
	
\end{figure*}

%Since the path received by the robot controller is equal to the path
%calculated by the inverse kinematic solver, the difference between
%them is only depends on the accuracy of the robot, which is maximum
%$\pm 0.02$ mm in operational space \cite{melfa_dat}. 
%However, the velocity profiles can have higher error.
Figure~\ref{fig:demo_velocity_error} shows the
difference between the measured and the reference velocity profile for the second joint,
while Figure~\ref{fig:demo_path_error} shows the joint position error
for the same joint.
Note that the tracking error is low despite the use of
such a simple controller.

%The velocity error can be reduced using a computed-torque control scheme with a torque control signal instead of the position signal. However, the presented Mitsubishi robot controller cannot receive torque movement command. 

\begin{figure}
	\centering
		\includegraphics[width=0.6\columnwidth, height = 10cm]{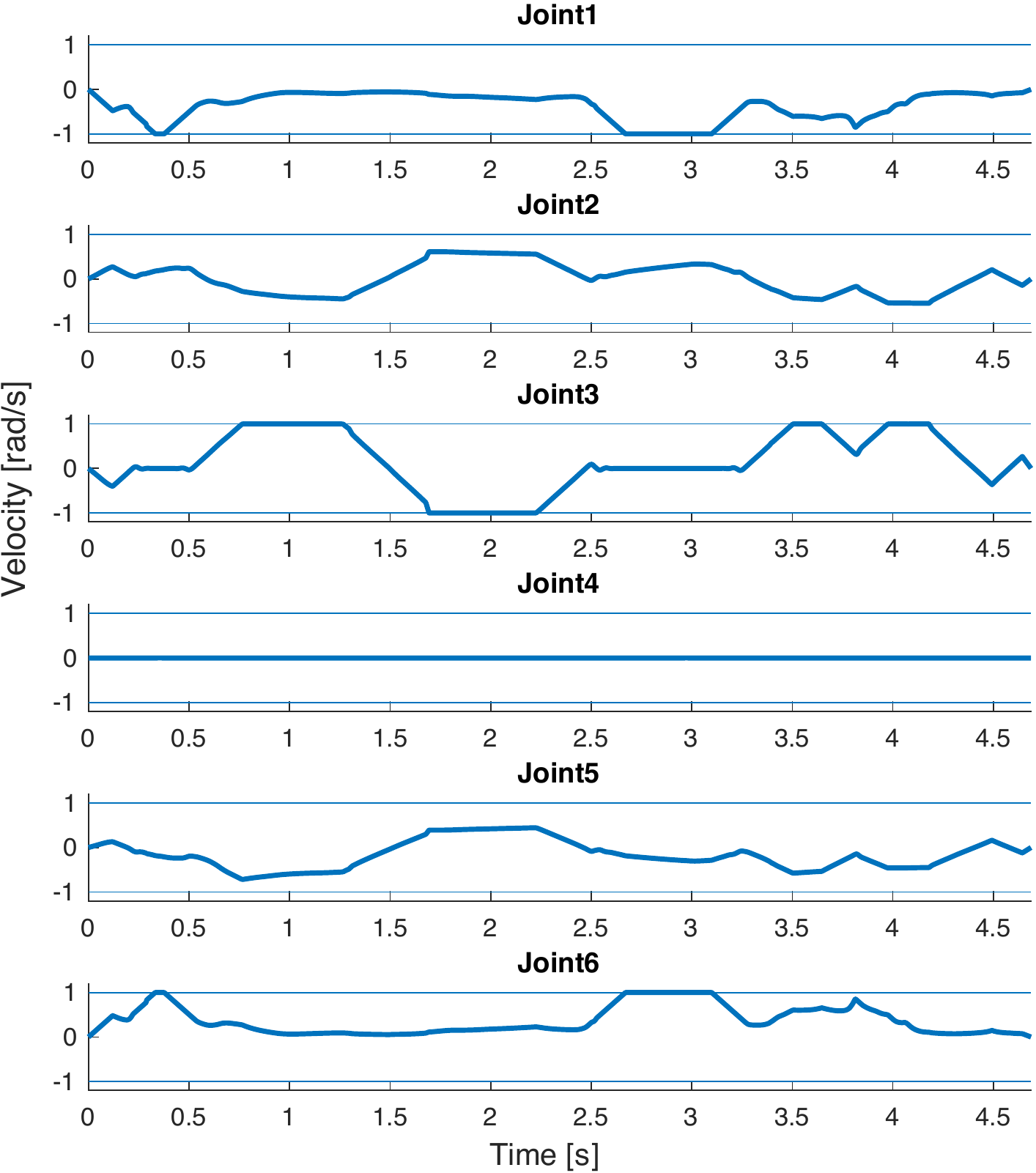}
		\caption{The generated profiles for the six-joints.} \label{fig:demo_profile}	
\end{figure}
\begin{figure}
	\centering
		\includegraphics[width=0.6\columnwidth]{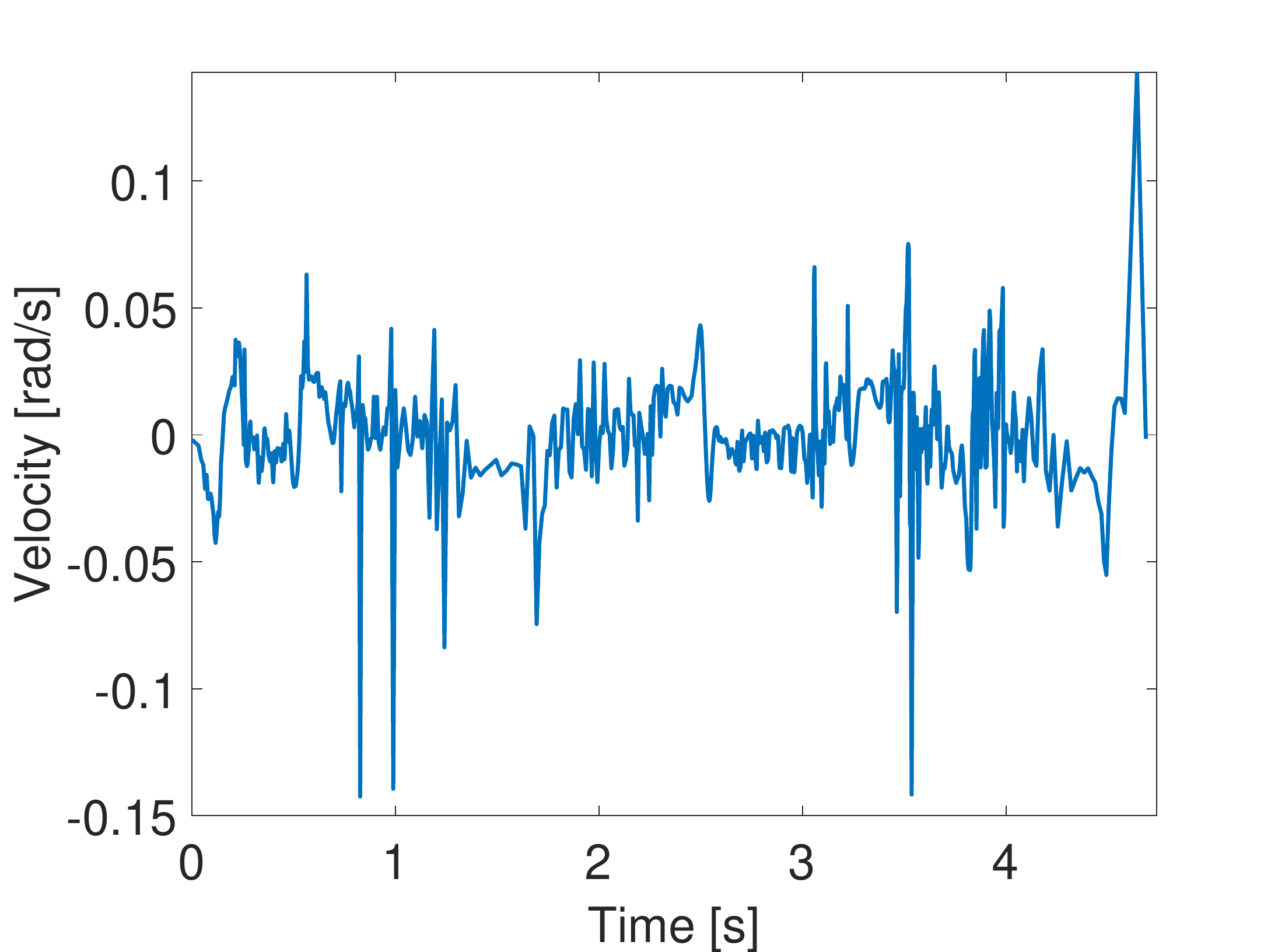}
		\caption{Measured velocity error for the second joint.} \label{fig:demo_velocity_error}	
\end{figure}
\begin{figure}
	\centering
		\includegraphics[width=0.6\columnwidth]{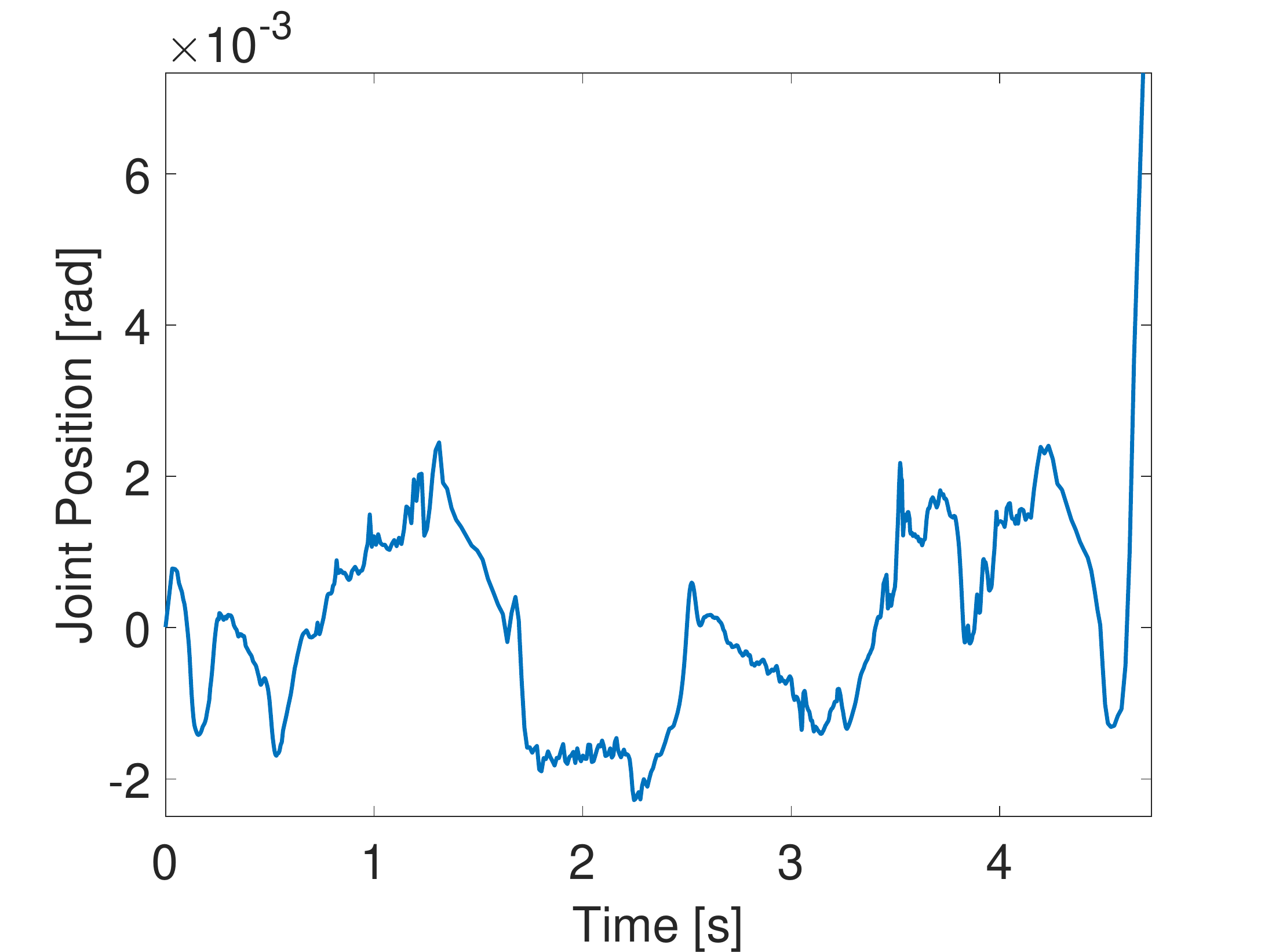}
		\caption{Measured joint position error for the second joint.} \label{fig:demo_path_error}	
\end{figure}

\section{Conclusions}\label{sec:conclusion}
We solved a speed planning problem for the robot manipulators taking into account velocity, acceleration and torque constraints.

We proposed an algorithm which solves a class of optimization problems and we showed that, 
in case of linear constraints, the complexity of such algorithm is optimal.

Using a suitable discretization strategy we proved that the speed planning problem for robotic manipulators falls in the class of problems we introduced and that can be solved using the proposed algorithm.

By numerical experiments, we showed that the proposed algorithm solves the speed planning problem much faster than the other solvers proposed in the literature.
Finally, we applied the proposed algorithm to control a real 6-DOF manipulator.

\section{Appendix}

\subsection{ Proof of Proposition \ref{prop-solution-existence}}

Let $\mathcal{D}$ be the subset of $C^{1}([0,s_{f}],\Real)$ that satisfies conditions
(\ref{con:dynamic-2})-(\ref{con:interp-2}). Set $f(b) = \int_0^{s_f} b(s)^{-1/2} \, ds$ and
 $f^{*} = \inf \{f(b) : b \in \mathcal{D} \} $, then there exists a sequence of
 $b_{i}:[0,s_{f}] \to \Real$, $i \in \mathbb{N}$, such that $b_{i} \in \mathcal{D}$ and
$ \lim_{i\to\infty} f(b_i) = f^{*}$.
By Ascoli-Arzel\`{a} theorem, if the sequence $\{b_{i}\}$ is uniformly bounded and differentiable with
$\{b_{i}^{\prime}\}$  uniformly bounded, then there exists a subsequence $\{b_{i_{k}}\}$ that uniformly converges on $[0,s_{f}]$.
Since $(\forall s \in [0,s_{f}])$ $\|\bgamma^{\prime}(s)\| = 1$, there exists an index $i(s) \in \{1,\dots,p\}$
such that $\gamma_{i(s)}^{\prime}(s)^{2} \geq \frac{1}{p}$.Then, we define a function $\beta : [0,s_{f}] \to \Real_{+}$ as the most restrictive  upper bound of $b$  along the path.
Hence, from constraint~(\ref{con:vel-2}) we can write the following relation
\begin{equation}
\label{eq:velocity-bound}
 0\le b(s) \le \beta(s).
\end{equation}

Since each function $b_{i}$ is uniformly bounded (by (\ref{con:vel-2}) and boundedness of $\beta$) and differentiable, it remains to show that $b_{i}^{\prime}$ is uniformly bounded (i.e., there exists a real constant $C$ such that, $\forall s \in [0,s_{f}] $, $\lvert b_{i}^{\prime}(s) \rvert \le C$). Consider constraint  (\ref{con:acc-2})
\[
-2(\balpha(s) +\bgamma^{\prime\prime}(s)b(s))  \le \bgamma^{\prime}(s)b^{\prime}(s)  \le 2(\balpha(s) - \bgamma^{\prime\prime}(s)b(s)).
\]

Since $\balpha$ and $\bbeta$ are bounded functions, $\bar\alpha = \lVert \balpha\rVert_{\infty} < +\infty$
and $\bar\beta = \| \bbeta \|_{\infty} < +\infty$.
Moreover,  $\bgamma^{\prime\prime}$ is a continuous function on the compact set
$[0,s_{f}]$, then there exists the component-wise
maximum $\bar\gamma^{\prime\prime} = \|\bgamma^{\prime\prime}\|_{\infty}$. Hence, we bound $\bgamma^{\prime}(s)b^{\prime}(s)$ as follows
\[
-2(\bar\alpha+\bar\gamma^{\prime\prime}\bar\beta)\be \le \bgamma^{\prime}(s)b^{\prime}(s) \le  2(\bar\alpha+\bar\gamma^{\prime\prime}\bar\beta)\be.
\]
Since $\forall s \in [0,s_{f}]$ $\| \bgamma^{\prime}(s)\|_{2} = 1$, then there exists an index $i(s)\in\{1,\dots,p\}$ such that  $\| \bgamma^{\prime}_{i(s)}(s) \|_{\infty} \ge \frac{1}{\sqrt{p}}$, which implies

\[
- 2\sqrt{p}(\bar\alpha+\bar\gamma^{\prime\prime}\bar\beta)\le-2\frac{(\bar{\alpha}+\bar{\gamma}^{\prime\prime}\bar{\beta})}{\|\gamma_{i(s)}^{\prime}(s)\|_{\infty}} \le b^{\prime}(s) \le  \frac{(\bar{\alpha}+\bar{\gamma}^{\prime\prime}\bar{\beta})}{\|\gamma_{i(s)}^{\prime}(s)\|_{\infty}}\le 2\sqrt{p}(\bar\alpha+\bar\gamma^{\prime\prime}\bar\beta).
\]
Hence, $|b^{\prime}|$ is uniformly bounded by the real constant $C = 2\sqrt{p}(\bar\alpha+\bar\gamma^{\prime\prime}\bar\beta)$.

To show that $f^{*}(b)\le U < +\infty$, where $U$ is a constant depending on the problem data,  it is
sufficient to find $b\in \mathcal{D}$ such that $f(b) < +\infty$. To this end, set for $\delta \ge 0$
\[
b_{\delta}(s) = \begin{cases}
\delta s (2\delta -s), & s\in[0,\delta),\\
\delta^{3}, & s \in [\delta,s_{f} -\delta],\\
\delta(s_{f} - s)(s- s_{f} + 2\delta) , & s \in (s_{f}-\delta,s_{f}].
\end{cases}
\]
Its derivative is
\[
b_{\delta}^{\prime}(s) = \begin{cases}
2\delta (\delta -s), & s\in[0,\delta),\\
0, & s \in [\delta,s_{f} -\delta],\\
2\delta(s_{f} - s - \delta) , & s \in (s_{f}-\delta,s_{f}].
\end{cases}
\]
Note that (\ref{con:interp-2}) is obviously satisfied by $b_{\delta}$, moreover,  $b_{\delta} \in \mathcal{D}$ if $\delta = 0$.
% Since $b_{\delta}$ is a continuous function, then $\forall \bar{s} \in [0,s_{f}] $ there exists an arbitrary
% small positive value $\delta(\bar{s})$ such that  constraints (\ref{con:dynamic-2})-(\ref{con:s-2}) are
% satisfied by $b_{\delta(\bar{s})}(s)$ for $s = \bar{s}$.  Since $[0, s_{f}]$ is a compact set, there exists the minimum $\bar\delta > 0 $, defined as  $\bar\delta = \min\{ \delta(\bar{s}) : \bar{s} \in [0,s_{f}]\}$,  such that $b_{\bar\delta}$ satisfies all constraints (\ref{con:dynamic-2})-(\ref{con:s-2}) for $s \in [0,s_{f}]$.

The maximum value of $b_{\delta}$ is $\delta^{3}$. 
By Assumption~\ref{ass:psi},  there exists the minimum 
$\hat{\psi} =\min_{i=1,\dots,p}\min\{ \bpsi_{i}(s) : s \in [0,s_{f}]\} >0$. Since $\bgamma \in C^{2}([0,s_{f}],\Real^{p})$ it follows that $\hat\gamma^{\prime}=\|\bgamma^{\prime}\|_{\infty}^{2} < +\infty $.
Hence, setting $\hat{\delta} = (\hat\psi / \hat\gamma^{\prime})^\frac{2}{3} $, ~(\ref{con:vel-2}) is satisfied for any 
$
\delta \in [0,\hat\delta]. 
$
After that, we have that $(\forall s \in [0,s_{f}])$
\begin{align*}
&\left|\frac{1}{2}\bd(s)b_{\delta}^{\prime}(s) + \bc(s)b_{\delta}(s) + \bg(s)\right| \le  \\
\le&\frac{1}{2}\left|\bd(s)\right| |b_{\delta}^{\prime}(s)| + |\bc(s)|b_{\delta}(s) + |\bg(s)|. 
\end{align*}
By Assumption~\ref{ass:mu} it follows that 
\begin{align*}
&\frac{1}{2}\left|\bd(s)\right| |b_{\delta}^{\prime}(s)| + |\bc(s)|b_{\delta}(s) + |\bg(s)| < \\
<&\frac{1}{2}\left|\bd(s)\right| |b_{\delta}^{\prime}(s)| + |\bc(s)|b_{\delta}(s) + \bmu(s) -\varepsilon\be.
\end{align*}
%By definition of $\bgamma \in C^{2}([0,s_{f}])$ relations (\ref{eq:dynamic_parameters}) implies that
%$\|\bd\|_{\infty} < +\infty$ and $\|\bc\|_{\infty} < +\infty$. 
%By Assumption \ref{ass:mu} $ \exists \varepsilon > 0$, $\varepsilon \in \Real$ such that  $(\forall s \in [0,s_{f}])$
%$\bmu(s) - |\bg(s)|  > \varepsilon \be$. 
There exists $\bar\delta > 0$ such that $(\forall s \in [0,s_f])$  
\[
 \frac{1}{2}|\bd(s)| |b^{\prime}_{\bar\delta}(s)| + |\bc(s)|b_{\bar\delta}(s) \le \varepsilon \be,
\]
%\[
%\left|\frac{1}{2}\bd(s)\right|b_{\bar\delta}^{\prime}(s) + |\bc(s)|b_{\bar\delta}(s) < \bmu(s)
%\]
which implies that for any $\delta \in [0,\bar\delta]$
%\[
%\lvert  |\frac{1}{2}\bd(s)\|b^{\prime}_{\bar\delta}(s) + |\bc(s)|b_{\bar\delta}(s) \rvert \le \varepsilon,
%\]
\[
\left|\frac{1}{2}\bd(s)b_{\delta}^{\prime}(s) + \bc(s)b_{\delta}(s) + \bg(s)\right| \le  \bmu(s),
\]
i.e., $b_{\delta}$ satisfies constraints (\ref{con:force-2}).
Analogously one can see that constraint (\ref{con:acc-2}) is fulfilled for each $\delta \in [0,\tilde\delta] $ with a sufficiently small $\tilde\delta > 0$. Hence, for each $\delta \in [0,\delta^{*}]$ with $\delta^{*} = \min\{\hat{\delta},\bar{\delta},\tilde{\delta}\}$, it follows  that $b_{\delta} \in \mathcal{D}$.
Finally, by direct computation, it is straightforward to see that $f(b_{\delta}) < +\infty $, with $\delta >0$.
\hfill$\square$

\subsection{Proof of Proposition~\ref{prop_interp_cond}}

For $i=2,\ldots,n-1$, $x_i$, $y_i$, $z_i$ need satisfy conditions
\begin{equation}
\label{eqn_conditions}
\begin{array}{ll}
x_i-\frac{h}{2} y_i + \frac{h^2}{4} z_i=b_{i-1/2},\\ [8pt]
x_i+\frac{h}{2} y_i + \frac{h^2}{4} z_i=b_{i+1/2},\\ [8pt]
y_i - h z_i=\delta_{i-1/2},\\ [8pt]
y_i + h z_i=\delta_{i+1/2}\,.
\end{array}
\end{equation}
Note that the last equation is redundant since it is a linear combination of the first
three (with coefficients -1, +1, and $-h/2$, respectively). Conditions~\eqref{eqn_conditions} can then be rewritten as
\begin{equation}
\label{eq:linsys}
M \vett{x_i\\y_i\\z_i}=
\vett{b_{i-1/2}\\b_{i+1/2}\\ \delta_{i-1/2}}\,,
\end{equation} where
\[
M=\left(
\begin{array}{rrr}
1 &-\frac{h}{2} & \frac{h^2}{4} \\
1 &\frac{h}{2}  &\frac{h^2}{4}   \\
0 & 1 &- h
\end{array}
\right)\,.
\]
The solution of~\eqref{eqn_conditions} (unique since
$M$ is nonsingular) is
\begin{equation}
\label{eq:sollinsys}
\begin{array}{l}
x_i=\displaystyle\frac{6 b_i+b_{i-1}+b_{i+1}}{8} \\ [8pt]
y_i=\displaystyle\frac{b_{i+1}-b_{i-1}}{2h} \\ [8pt]
z_i=\displaystyle\frac{b_{i+1}+b_{i-1}- 2 b_i}{2 h^2}.
\end{array}
\end{equation}
Moreover $x_1$, $y_1$, $z_1$ need satisfy
\[
\begin{array}{ll}
x_1=b_1,\\
x_1+\frac{h}{2} y_1 + \frac{h^2}{4} z_1=b_{1+1/2},\\
  y_1 + h z_1=\delta_{1+1/2}\,,
\end{array}
\]
whose solution is unique and is given by $x_1=b_1$, $z_1=0$,
$y_1=\frac{b_{2}-b_{1}}{2 h}$. Finally,
$x_n$, $y_n$, $z_n$ need satisfy
\[
\begin{array}{ll}
x_n=b_n,\\
x_n-\frac{h}{2} y_n + \frac{h^2}{4} z_n=b_{n-1/2},\\
y_n  + h z_n=\delta_{n-1/2}\,,
\end{array}
\]
whose solution is unique and is given by $x_n=b_n$, $z_n=0$,
$y_n=\frac{b_{n}-b_{n-1}}{h}$.
\hfill $\square$

\subsection{Proof of Proposition \ref{prop:selection}}
Since the objective function (\ref{con:obj-4}) is monotonic non increasing and the variables $b_i$, $i=0,\dots,n$, are non negative and bounded by (\ref{eq:velocity-bound}), we only need to prove that  constraints (\ref{con:acc-4}) and (\ref{con:force-4})  satisfy Assumption \ref{ass:1} for suitable choices  of $\lambda_{j,i}$ and $\eta_{i,j}$, $j = 1,\dots,p$  and $i = 1,\dots,n $.

For the sake of simplicity consider the $j$-th component of the $i$-th sample of (\ref{con:force-4}).
Substituting variable $a_{i}$ and $\tau_{i}$ with (\ref{con:approx-4}) and (\ref{con:dynamic-4}) in constraints (\ref{con:force-4}) we have:
\[
| (d_{j,i} + 2h c_{j,i} \lambda_{j,i})b_{i+1} 
+ (-d_{j,i}  + (1-\lambda_{j,i})2h c_{j,i})b_{i} +  2hg_{j,i} | \le 2h \mu_{j,i}
\]
First, we discuss the cases when $d_{j,i} = 0 $ or $c_{j,i} = 0$. For these cases we set $\lambda_{j,i} = 1$.
If $d_{j,i} = c_{j,i} = 0$, we have $\lvert g_{j,i} \rvert < \mu_{j,i}$ that is always true by Assumption \ref{ass:mu}.
If $d_{j,i} = 0$ and $c_{j,i} \ne 0$ we have $\lvert c_{j,i} b_{i+1} + g_{j,i} \lvert \le \mu_{j,i}$ that,
 after combining it with constraints (\ref{con:vel-4}), becomes $ 0\le b_{i+1} \le \min\{\psi_{i+1}^2/\bgamma_{j,i+1}^{\prime 2},(\mu_{j,i} - g_{j,i})/ \lvert c_{j,i} \rvert \} $. Finally,  if $d_{j,i} \ne 0 $ and $c_{j,i} = 0$ we have $\lvert d_{j,i}(b_{i+1} - b_{i})  +2hg_{j,i} \rvert \le 2h\mu_{j,i} $ that satisfies Assumption \ref{ass:1}.
If $d_{j,i} \ne 0$ and $c_{j,i} \ne 0$  we have:
\[
-2h(\mu_{j,i}+g_{j,i}) \le  (d_{j,i} + 2h c_{j,i} \lambda_{j,i})b_{i+1} + (-d_{j,i} + (1-\lambda_{j,i})2h c_{j,i})b_{i} \le 2h(\mu_{j,i} - g_{j,i}).
\]
In order to satisfy Assumption \ref{ass:1} we choose the value of $\lambda_{j,i}$ such that $(d_{j,i} + 2h c_{j,i} \lambda_{j,i})(-d_{j,i} + (1-\lambda_{j,i})2h c_{j,i})< 0$. Hence, we set
\[
\lambda_{j,i} =
\begin{cases}
1 & \mbox{if } \ d_{j,i}c_{j,i} > 0 \\
0 & \mbox{if } \ d_{j,i}c_{j,i}  < 0.
\end{cases}
\]
Using  this selection technique we can rewrite constraint (\ref{con:dynamic-4}) as follows:
\[
\tau_{j,i} =
\begin{cases}
d_{j,i}a_{i} +c_{j,i}b_{i+1} + g_{j,i}, & \mbox{if} \ d_{j,i}\ c_{j,i} > 0, \\
d_{j,i}a_{i} +c_{j,i}b_{i} + g_{j,i}, &  \mbox{if} \ d_{j,i} \ c_{j,i}  < 0.
\end{cases}
\]
Moreover, we can explicit constraints (\ref{con:force-4}) in the form presented in (\ref{eq:probl}). In fact, if~$d_{j,i} \ c_{j,i} > 0$, the constraint (\ref{con:force-4}) becomes:
\[
\begin{array}{l}
\displaystyle b_{i+1} \le \frac{d_{j,i}}{d_{j,i} + 2hc_{j,i}}b_{i} + \left|  \frac{2h(\mu_{j,i} - g_{j,i})}{d_{j,i} + 2hc_{j,i}}\right|, \\ [12pt]
\displaystyle b_{i} \le \frac{d_{j,i} + 2hc_{j,i}}{d_{j,i}}b_{i+1} + \left| \frac{2h(\mu_{j,i} + g_{j,i})}{d_{j,i} }\right|
\end{array}
\]
and, with  $d_{j,i} \ c_{j,i} < 0 $,
\[
\begin{array}{l}
\displaystyle b_{i+1} \le \frac{d_{j,i} - 2hc_{j,i}}{d_{j,i}}b_{i} + \left| \frac{2h(\mu_{j,i} + g_{j,i})}{d_{j,i} }\right| \\ [12pt]
\displaystyle b_{i} \le \frac{d_{j,i}}{d_i - 2hc_{j,i}}b_{i+1} + \left|  \frac{2h(\mu_{j,i} - g_{j,i})}{d_{j,i} - 2hc_{j,i}}\right|, \\
\end{array}
\]
which satisfy Assumption \ref{ass:1}.

We use the same reasoning  for  constraints (\ref{con:acc-4}). Consider
\[
| (\gamma^{\prime}_{j,i} + 2h \gamma_i^{\prime\prime}\eta_{j,i})b_{i+1}  
 + (-\gamma^{\prime}_{j,i} + (1-\eta_{j,i})2h \gamma^{\prime\prime}_{j,i})b_{i} | \le 2h \alpha_{j,i}
\]

Again, setting $\eta_{j,i} = 1$,
we discuss the cases when $\gamma^{\prime}_{j,i} = 0 $ or $\gamma^{\prime\prime}_{j,i} = 0$.
If $\gamma^{\prime}_{j,i} = \gamma^{\prime\prime}_{j,i} = 0$, we have $\lvert 0 \rvert \le \alpha_{j,i}$ that is always
true. If $\gamma^{\prime}_{j,i} = 0$ and $\gamma^{\prime\prime}_{j,i} \ne 0$ we have $\lvert \gamma^{\prime\prime}_{j,i} b_i \lvert \le \alpha_{j,i}$ that becomes $ 0\le b_{i+1} \le \alpha_{j,i}/ \lvert \gamma^{\prime\prime}_{j,i} \rvert  $ . Finally,  if $\gamma_{j,i}^{\prime} \ne 0 $ and $\gamma^{\prime\prime}_{j,i} = 0$ we have $\lvert b_{i+1} - b_{i}  \rvert \le 2h\alpha_{j,i} / \lvert \gamma^{\prime}_{j,i} \rvert$ that satisfies Assumption \ref{ass:1}.
After that, with $\gamma_{j,i}^{\prime} \ne 0$ and $\gamma_{j,i}^{\prime\prime}\ne0$ we set
\[
\eta_{j,i} =
\begin{cases}
1 & \mbox{if} \ \gamma_{j,i}^{\prime} \ \gamma^{\prime\prime}_{j,i} > 0\\
0 & \mbox{if} \ \gamma_{j,i}^{\prime} \ \gamma^{\prime\prime}_{j,i} < 0
\end{cases},
\]
which implies, for $\gamma_{j,i}^{\prime}\cdot \gamma_{j,i}^{\prime\prime} > 0$:
\[
\begin{array}{l}
\displaystyle b_{i+1} \le \frac{\gamma^{\prime}_{j,i}}{\gamma^{\prime}_{j,i} + 2h\gamma^{\prime\prime}_{j,i}}b_{i} + \left|  \frac{2h\alpha_{j,i}}{\gamma^{\prime}_{j,i} + 2h\gamma^{\prime\prime}_{j,i}}\right| \\ [12pt]
\displaystyle b_{i} \le \frac{\gamma^{\prime}_{j,i} + 2h\gamma^{\prime\prime}_{j,i}}{\gamma^{\prime}_{j,i}}b_{i+1} + \left| \frac{2h\alpha_{j,i}}{\gamma^{\prime}_{j,i} }\right|
\end{array},
\]
and, with  $\gamma^{\prime}_{j,i} \cdot \gamma^{\prime\prime}_{j,i} < 0 $,
\[
\begin{array}{l}
\displaystyle b_{i+1} \le \frac{\gamma^{\prime}_{j,i} - 2h\gamma^{\prime\prime}_{j,i}}{\gamma^{\prime}_{j,i}}b_{i} + \left| \frac{2h\alpha_{j,i}}{\gamma_{j,i}^{\prime} }\right|\\
\displaystyle b_{i} \le \frac{\gamma^{\prime}_{j,i}}{\gamma^{\prime}_{j,i} - 2h\gamma^{\prime\prime}_{j,i}}b_{i+1} + \left|  \frac{2h\alpha_{j,i}}{\gamma^{\prime}_{j,i} - 2h\gamma^{\prime\prime}_{j,i}}\right| \\ [12pt]
\end{array}.
\]
which satisfy Assumption \ref{ass:1}.

\hfill$\square$

% or


\bibliographystyle{IEEEtranN}
\bibliography{VelPlan}

\end{document}